\documentclass[11pt]{article} %
\usepackage{times}
\usepackage{url}            %
\usepackage{booktabs}       %
\usepackage{amsfonts}       %
\usepackage{nicefrac}       %
\usepackage{microtype}      %
\usepackage{mkolar_definitions}
\usepackage{xspace}
\usepackage{multirow}
\usepackage{amsmath}
\usepackage{algorithm}
\usepackage{algorithmic}
\usepackage{color}
\usepackage{MnSymbol}
\usepackage{color, colortbl}
\usepackage{enumitem}
\usepackage{comment}
\usepackage{bm}
\usepackage[left=2cm,top=2cm,right=2cm]{geometry}
\usepackage{authblk}
\usepackage[scaled=.9]{helvet}

\usepackage{url}

\usepackage{authblk}
\usepackage{csquotes}

\usepackage[dvipsnames]{xcolor}
\usepackage[colorlinks=true, linkcolor=blue, citecolor=blue]{hyperref}

\usepackage{natbib}
\bibpunct{(}{)}{;}{a}{,}{,}

\usepackage{centernot}

\newcount\Comments  %
\definecolor{darkgreen}{rgb}{0,0.5,0}
\definecolor{darkred}{rgb}{0.7,0,0}
\definecolor{teal}{rgb}{0.3,0.8,0.8}
\definecolor{orange}{rgb}{1.0,0.5,0.0}
\definecolor{purple}{rgb}{0.8,0.0,0.8}
\newcommand{\kibitz}[2]{\ifnum\Comments=1{\textcolor{#1}{\textsf{\footnotesize #2}}}\fi}

\definecolor{Gray}{gray}{0.9}

\newcommand{\Mid}{\mathrel{\Vert}}

\newcommand{\op}{\operatorname}
\newcommand{\DE}{\op{DE}}
\newcommand{\dimRL}{d_{\normalfont\textsf{RL}}}

\begin{document}
\title{Model-based RL as a Minimalist Approach to Horizon-Free and Second-Order Bounds}


\author[1]{Zhiyong Wang\thanks{zywang21@cse.cuhk.edu.hk}}
\author[2]{Dongruo Zhou\thanks{dz13@iu.edu}}
\author[1]{John C.S. Lui\thanks{cslui@cse.cuhk.edu.hk}}
\author[3]{Wen Sun\thanks{ws455@cornell.edu}}
\affil[1]{The Chinese University of Hong Kong}
\affil[2]{Indiana University Bloomington}
\affil[3]{Cornell University}
\maketitle

\begin{abstract}
Learning a transition model via Maximum Likelihood Estimation (MLE) followed by planning inside the learned model is perhaps the most standard and simplest Model-based Reinforcement Learning (RL) framework. In this work, we show that such a simple Model-based RL scheme, when equipped with optimistic and pessimistic planning procedures, achieves strong regret and sample complexity bounds in online and offline RL settings. Particularly, we demonstrate that under the conditions where the trajectory-wise reward is normalized between zero and one and the transition is time-homogenous, it achieves nearly horizon-free and second-order bounds. Nearly horizon-free means that our bounds have no polynomial dependence on the horizon of the Markov Decision Process. A second-order bound is a type of instance-dependent bound that scales with respect to the variances of the returns of the policies which can be small when the system is nearly deterministic and (or) the optimal policy has small values. We highlight that our algorithms are simple, fairly standard, and indeed have been extensively studied in the RL literature: they learn a model via MLE, build a version space around the MLE solution, and perform optimistic or pessimistic planning depending on whether operating in the online or offline mode. These algorithms do not rely on additional specialized algorithmic designs such as learning variances and performing variance-weighted learning and thus can easily leverage non-linear function approximations. The simplicity of the algorithms also implies that our horizon-free and second-order regret analysis is actually standard and mainly follows the general framework of optimism/pessimism in the face of uncertainty.
\end{abstract}

\vspace{-0.16cm}
\section{Introduction} 
\label{sec:intro}

\vspace{-0.26cm}
The framework of model-based Reinforcement Learning (RL) often consists of two steps: fitting a transition model using data and then performing planning inside the learned model. Such a simple framework turns out to be powerful and has been used extensively in practice on applications such as robotics and control (e.g., \citep{aboaf1989task,deisenroth2011learning,venkatraman2017improved,williams2017information,chua2018deep,kaiser2019model,yang2023learning}).  \looseness=-1

The simplicity of model-based RL also attracts researchers to analyze its performance in settings such as online RL \citep{sun2019model} and offline RL \citep{uehara2021pessimistic}. \cite{mania2019certainty} showed that this simple scheme --- fitting model via data followed by optimal planning inside the model, has a strong performance guarantee under the classic linear quadratic regulator (LQR) control problems.  \cite{liu2023optimistic} showed that this simple MBRL framework when equipped with optimism in the face of the uncertainty principle, can achieve strong sample complexity bounds for a wide range of online RL problems with rich function approximation for the models. For offline settings where the model can only be learned from a static offline dataset, \cite{uehara2021pessimistic} showed that MBRL equipped with the pessimism principle can again achieve robust performance guarantees for a large family of MDPs. \cite{ross2012agnostic} showed that in the hybrid RL setting where one has access to both online and offline data, this simple MBRL framework again achieves favorable performance guarantees without any optimism/pessimism algorithm design. \looseness=-1

In this work, we do not create new MBRL algorithms, instead, we show that the extremely simple and standard MBRL algorithm -- fitting models using Maximum Likelihood Estimation (MLE), followed by optimistic/pessimistic planning (depending on whether operating in online RL or offline RL mode), can already achieve surprising theoretical guarantees. Particularly, we show that under the conditions that trajectory-wise reward is normalized between zero and one, and the transition is time-homogenous, they can achieve nearly horizon-free and instance-dependent regret and sample complexity bounds, in both online and offline RL with non-linear function approximation. Nearly horizon-free bounds mean that the regret or sample complexity bounds have no explicit polynomial dependence on the horizon $H$. The motivation for studying horizon-free RL is to see if RL problems are harder than bandits due to the longer horizon planning in RL. \emph{Our result here indicates that, even under non-linear function approximation, long-horizon planning is not the bottleneck of achieving statistical efficiency in RL.}
For instance-dependent bounds, we focus on second-order bounds.  A second-order regret bound scales with respect to the variances of the returns of policies and also directly implies a first-order regret bound which scales with the expected reward of the optimal policy. Thus our instance-dependent bounds can be small under situations such as nearly-deterministic systems or the optimal policy having a small value. 
When specializing to the case of deterministic ground truth transitions (but the algorithm does not need to know this a priori), we show that these simple MBRL algorithms demonstrate a faster convergence rate than the worst-case rates.  
The key message of our work is \looseness=-1
\vspace{-0.1cm}
\begin{center}
    \emph{Simple and standard MLE-based MBRL algorithms are sufficient for achieving nearly horizon-free and second-order bounds in online and offline RL with function approximation. }
\end{center}
\vspace{-0.1cm}

We provide a fairly standard analysis to support the above claim. Our analysis follows the standard frameworks of optimism/pessimism in the face of uncertainty.
For online RL. we use $\ell_1$ Eluder dimension \citep{liu2022partially,wang2024more}, a condition that uses both the MDP structure and the function class, to capture the structural complexity of exploration. For offline RL, we use the similar concentrability coefficient in \cite{ye2024corruption} to capture the coverage condition of the offline data. The key technique we leverage is the \emph{triangular discrimination} -- a divergence that is equivalent to the squared Hellinger distance up to some universal constants. Triangular discrimination was used in contextual bandit and model-free RL for achieving first-order and second-order instance-dependent bounds \citep{foster2021efficient,wang2023benefits,wang2024more}. Here we show that it also plays an important role in achieving horizon-free bounds. Our contributions can be summarized as follows.
\vspace{-0.26cm}
\begin{enumerate}[leftmargin = *]
\item Our results extend the scope of the prior work on horizon-free RL which only applies to tabular MDPs or MDPs with linear functions. 
Given a finite model class $\Pcal$ (which could be exponentially large), we show that in online RL, the agent achieves an {\small$O\left(  \sqrt{  ( \sum_k \var_{\pi^k} ) \cdot  d_{\text{RL}}  \log ( KH | \Pcal | / \delta )   } + d_{\text{RL}} \log ( KH |\Pcal|/\delta)  \right)$} regret, where $K$ is the number of episodes, $d_{\text{RL}}$ is the $\ell_1$ Eluder dimension, $\var_{\pi^k}$ is the variance of the total reward of policy $\pi^k$ learned in episode $k$ and $\delta \in (0,1)$ denotes the failure probability. Similarly, for offline RL, the agent achieves an {\small$O\left(\sqrt{{C^{\pi^*} \var_{\pi^*}\log(|\Pcal|/\delta)}/{K} } + {C^{\pi^*}\log(|\Pcal|/\delta)}/{K}\right)$} performance gap in finding a comparator policy $\pi^*$, where $C^{\pi^*}$ is the single policy concentrability coefficient over $\pi^*$, $K$ denotes the number of offline trajectories, $\var_{\pi^*}$ is the variance of the total reward of $\pi^*$. For offline RL with finite $\Pcal$, our result is \emph{completely horizon-free}, not even with $\log H$ dependence.

\item When specializing to MDPs with deterministic ground truth transition (but rewards, and models in the model class could still be stochastic), we show that the same simple MBRL algorithms can adapt to the deterministic environment and achieve a better statistical complexity. For online RL, the regret becomes $O(d_{\text{RL}} \log( K H|\Pcal|/\delta) )$, which only depends on the number of episodes $K$ poly-logarithmically. For offline RL, the performance gap to a comparator policy $\pi^*$ becomes {\small$O\left( { C^{\pi^*} \log( |\Pcal | /\delta ) }/{K }\right)$}, which is tighter than the worst-case $O(1/\sqrt{K})$ rate. All our results can be extended to continuous model class $\Pcal$ using bracket number as the complexity measure. \looseness=-1


\end{enumerate}
\vspace{-0.16cm}
Overall, our work identifies the \emph{minimalist} algorithms and analysis for nearly horizon-free and instance-dependent (first \& second-order) online \& offline RL.

\vspace{-0.26cm}
\section{Related Work}

\vspace{-0.16cm}
\paragraph{Model-based RL.} Learning transition models with function approximation and planning with the learned model is a standard approach in RL and control. In the control literature, certainty-equivalence control learns a model from some data and plans using the learned model, which is simple but effective for controlling systems such as Linear Quadratic Regulators (LQRs) \citep{mania2019certainty}.  In RL, such a simple model-based framework has been widely used in theory with rich function approximation, for online RL \citep{sun2019model,foster2021statistical,song2021pc,zhan2022pac,liu2022partially, liu2023optimistic,zhong2022gec}, offline RL \citep{uehara2021pessimistic}, RL with representation learning \citep{agarwal2020flambe, uehara2021representation}, and hybrid RL using both online and offline data for model fitting \citep{ross2012agnostic}. Our work builds on the maximum-likelihood estimation (MLE) approach, a standard method for estimating transition models in model-based RL.
\vspace{-0.16cm}
\paragraph{Horizon-free and Instance-dependent bounds.} 
Most existing works on horizon-free RL typically focus on tabular settings or linear settings. For instance, \cite{wang2020long} firstly studied horizon-free RL for tabular MDPs and proposed an algorithm that depends on horizon logarithmically. Several follow-up work studied horizon-free RL for tabular MDP with better sample complexity \citep{zhang2021reinforcement}, offline RL \citep{ren2021nearly}, stochastic shortest path \citep{tarbouriech2021stochastic} and RL with linear function approximation \citep{kim2022improved, zhang2021improved, zhou2022computationally, di2023nearly, zhang2024horizon, zhang2023optimal, zhao2023variance}. Note that all these works have logarithmic dependence on the horizon $H$. For the tabular setting, recent work further improved the regret or sample complexity to be completely independent of the horizon (i.e., removing the logarithmic dependence on the horizon) \citep{li2022settling, zhang2022horizon} with a worse dependence on the cardinality of state and action spaces $|\mathcal{S}|$ and $|\mathcal{A}|$. To compare with this, we show that simple MBRL algorithms are already enough to achieve completely horizon-free (i.e., no log dependence) sample complexity for offline RL when the transition model class is finite, and we provide a simpler approach to achieve the nearly horizon-free results for tabular MDPs, compared with \cite{zhang2021reinforcement}. A recent work \citep{huang2024horizon} also studied the horizon-free and instance-dependent online RL in the function approximation setting with small Eluder dimensions. They estimated the variances to conduct variance-weighted regression. To compare, in our online RL part, we use the simple and standard MLE-based MBRL approach and analysis to get similar guarantees. A more recent work also studied horizon-free behavior cloning \cite{foster2024behavior}, which is different from our settings.


Besides horizon-free RL, another line of work aimed to provide algorithms with instance-dependent sample complexity/regret bounds, which often enjoy tighter statistical complexity compared with previous work. To mention a few, \cite{zanette2019tighter} proposed an EULER algorithm with an instance-dependent regret which depends on the maximum variance of the policy return over all policies. Later, \cite{wagenmaker2022first, wang2023benefits} proposed algorithms with first-order regret bounds. A more refined second-order regret bound has been studied. The second-order regret bound is a well-studied instance-dependent bound in the online learning and bandit literature \citep{cesa2007improved,ito2020tight,olkhovskaya2024first}, and compared to the bounds that depend on the maximum variance over all policies, it can be much smaller and it also implies a first-order regret bound. \cite{zhang2024settling, zhou2023sharp} proposed algorithms for tabular MDP with second-order regret bounds. \cite{zhao2023variance} studied the RL with linear function approximation and proposed algorithms with both horizon-free and variance-dependent regret bounds. The closest work to us is \cite{wang2023benefits,wang2024more}, which used model-free distributional RL methods to achieve first-order and second-order regret bounds in RL. Their approach relies on distributional RL, a somewhat non-conventional approach for RL. Their regret bounds have explicit polynomial dependence on the horizon. Our work focuses on the more conventional model-based RL algorithms and demonstrates that they are indeed sufficient to achieve horizon-free and second-order regret bounds.  \looseness=-1



 



\vspace{-0.26cm}
\section{Preliminaries}
\vspace{-0.16cm}
\paragraph{Markov Decision Processes.} We consider finite horizon time-homogenous MDP $\Mcal = \{\Scal, \Acal, H, P^\star, r, s_0\}$ where $\Scal, \Acal$ are the state and action space (could be large or even continuous), $H\in \NN^+$ is the horizon, $P^\star: \Scal\times\Acal\mapsto \Delta(\Scal)$ is the ground truth transition, $r:\Scal\times\Acal\mapsto \RR$ is the reward signal which we assume is known to the learner, and $s_0$ is the fixed initial state.\footnote{For simplicity, we assume initial state $s_0$ is fixed and known. Our analysis can be easily extended to a setting where the initial state is sampled from an unknown fixed distribution.} Note that the transition $P^\star$ here is time-homogenous. For notational easiness, we denote $[K-1]=\{0,1,\ldots, K-1\}$.

 We denote $\pi$ as a deterministic non-stationary policy $\pi = \{ \pi_0, \dots, \pi_{H-1} \}$  where $\pi_h: \Scal\mapsto \Acal$ maps from a state to an action. Let $\Pi$ denote the set of all such policies. $V^\pi_h(s)$ represents the expected total reward of policy $\pi$ starting at $s_h = s$, and $Q^\pi_h(s,a)$ is the expected total reward of the process of executing $a$ at $s$ at time step $h$ followed by executing $\pi$ to the end. The optimal policy $\pi^\star$ is defined as $\pi^\star = \argmax_{\pi} V^{\pi}_0(s_0)$.  For notation simplicity, we denote $V^\pi := V^{\pi}_0(s_0)$. We will denote $d_h^{\pi}(s,a)$ as the state-action distribution induced by policy $\pi$ at time step $h$. We sometimes will overload notation and denote $d^\pi_h(s)$ as the corresponding state distribution at $h$. Sampling $s\sim d^\pi_h$ means executing $\pi$ starting from $s_0$ to $h$ and returning the state at time step $h$. 

Since we use the model-based approach for learning, we define a general model class $\Pcal \subset \Scal\times\Acal\mapsto \Delta(\Scal)$. 
Given a transition $P$, we denote $V_{h; P}^\pi$ and $Q_{h;P}^\pi$ as the value and Q functions of policy $\pi$ under the model $P$.  
Given a function $f:\Scal\times\Acal\mapsto \RR$, we denote the $(P f)(s,a) := \EE_{s'\sim P(s,a)} f(s')$.  We then denote the \emph{variance induced by one-step transition $P$ and  function $f$} as $(\VV_P f)(s,a):=  \left( P f^2 \right)(s,a)  - \left( P f(s,a) \right)^2$ which is equal to $\EE_{s'\sim P(s,a)} f^2(s') - \left( \EE_{s'\sim P(s,a)} f(s')  \right)^2$. 

\vspace{-0.16cm}
\paragraph{Assumptions.} We make the realizability assumption that $P^\star \in \Pcal$. We assume that the rewards are normalized such that $r(\tau) \in [0,1]$ for any trajectory $\tau:= \{s_0,a_0,\dots, s_{H-1},a_{H-1}\}$ where $r(\tau)$ is short for $\sum_{h=0}^{H-1} r(s_h,a_h)$. Note that this setting is more general than assuming each one-step reward is bounded, i.e., $r(s_h,a_h) \in [0,1/H]$, and allows to represent the sparse reward setting. Without loss of generalizability, we assume $V^{\pi}_{h;P}(s)\in[0,1]$, for all $\pi\in \Pi, h\in[0,H], P\in\Pcal, s\in\Scal$\footnote{$r(\tau) \in [0,1]$ implies $V^{\pi}_{h;P^\star}(s)\in[0,1]$. If we do not assume $V^{\pi}_{h;P}(s)\in[0,1]$ for all $P\in\Pcal$, we can simply add a filtering step in the algorithm to only choose $\pi$,$P$ with $V^{\pi}_{h;P}(s_0)\in[0,1]$ to get the same guarantees.}. \looseness=-1
\vspace{-0.16cm}
\paragraph{Online RL.}
For the online RL setting, we focus on the episodic setting where the learner can interact with the environment for $K$ episodes. At episode $k$, the learner proposes a policy $\pi^k$ (based on the past interaction history),  executes $\pi^k$ starting from $s_0$ to time step $H-1$. We measure the performance of the online learning via \emph{regret}: $\sum_{k=0}^{K-1} \left( V^{\pi^\star} - V^{\pi^k} \right)$. To achieve meaningful regret bounds, we often need additional structural assumptions on the MDP and the model class $\Pcal$. We use a $\ell_1$ Eluder dimension \citep{liu2022partially} as the structural condition due to its ability to capture non-linear function approximators (formal definition will be given in \pref{sec:online}). 
\vspace{-0.16cm}
\paragraph{Offline RL.}
For the offline RL setting, we assume that we have a pre-collected offline dataset $\Dcal = \{\tau^{i}\}_{i=1}^K$ which contains $K$ trajectories. For each trajectory, we allow it to potentially be generated by an adversary, i.e., at step $h$ in trajectory $k$, (i.e., $s_h^k$), the adversary can select $a^k_h$ based on all history (the past $k-1$ trajectories and the steps before $h$ within trajectory $k$) with a fixed strategy, with the only condition that the state transitions follow the underlying transition dynamics, i.e., $s^i_{h+1} \sim P^\star(s_h^i,a_h^i)$. We emphasize that $\Dcal$ is not necessarily generated by some offline trajectory distribution. 
Given $\Dcal$, we can split the data into $H K$ many state-action-next state $(s,a,s')$ tuples which we can use to learn the transition. To succeed in offline learning, we typically require the offline dataset to have good coverage over some high-quality comparator policy $\pi^*$ (formal definition of coverage will be given in \pref{sec:offline}). Our goal here is to learn a policy $\widehat\pi$ that is as good as $\pi^*$, and we are interested in the \emph{performance gap} between $\hat \pi$ and $\pi^*$, i.e., $V^{\pi^*} - V^{\hat \pi}$.


\vspace{-0.16cm}
\paragraph{Horizon-free and Second-order Bounds.} Our goal is to achieve regret bounds (online RL) or performance gaps (offline RL) that are (nearly) horizon-free, i.e., logarithmical dependence on $H$. In addition to the horizon-free guarantee, we also want our bounds to scale with respect to the variance of the policies. Denote $\var_\pi$ as the variance of trajectory reward, i.e., $\var_\pi := \EE_{\tau \sim \pi} (  r(\tau) - \EE_{\tau\sim\pi} r(\tau) )^2$.  Second-order bounds in offline RL scales with $\var_{\pi^*}$ -- the variance of the comparator policy. Second-order regret bound in online setting scales with respect to $\sqrt{ \sum_k \var_{\pi^k}  }$ instead of $\sqrt{K}$.
Note that in the worst case, $\sqrt{ \sum_k \var_{\pi^k}  }$ scales in the order of $\sqrt{K}$, but can be much smaller in benign cases such as nearly deterministic MDPs.  We also note that second-order regret bound immediately implies first-order regret bound in the reward maximization setting, which scales in the order $\sqrt{ K V^{\pi^\star}  }$ instead of just $\sqrt{K}$. The first order regret bound $\sqrt{ K V^{\pi^\star} }$ is never worse than $\sqrt{K}$ since $V^{\pi^\star} \leq 1$. Thus, by achieving a second-order regret bound, our algorithm immediately achieves a first-order regret bound. \looseness=-1
\vspace{-0.16cm}
\paragraph{Additional notations.} Given two distributions $p \in \Delta(\Xcal)$ and $q \in \Delta(\Xcal)$, we denote the triangle discrimination  $D_\triangle( p \Mid q ) =  \sum_{x\in\Xcal} \frac{ (p(x) - q(x))^2 }{ p(x) + q(x) } $, and squared Hellinger distance $\mathbb H^2( p \Mid q) = \frac{1}{2} \sum_{x\in \Xcal} \left( \sqrt{ q(x)} - \sqrt{ p(x)} \right)^2 $ (we replace sum via integral when $\Xcal$ is continuous and $p$ and $q$ are pdfs). Note that $D_\triangle$ and $\mathbb H^2$ are equivalent up to universal constants. We will frequently use the following key lemma in \citep{wang2024more} to control the difference between means of two distributions. 
\begin{lemma}[Lemma 4.3 in \cite{wang2024more}]\label{lem: mean to variance}
For two distributions $f \in \Delta([0,1])$ and $g \in \Delta([0,1])$:\looseness=-1
\begin{equation}
    \abs{\EE_{x\sim f}[x] - \EE_{x\sim g} [x] }\leq 4\sqrt{ \var_f \cdot D_\triangle(f\Mid g) } + 5D_\triangle(f\Mid g). \label{eq:var-key-ineq2}
\end{equation} where $\var_f := \EE_{x\sim f} ( x - \EE_{x\sim f}[x])^2$ denotes the variance of the distribution $f$.
\end{lemma}
The lemma plays a key role in achieving second-order bounds \citep{wang2024more}. The intuition is the means of the two distributions can be closer if one of the distributions has a small variance. A more naive way of bounding the difference in means is $|\EE_{x\sim f}[x] - \EE_{x\sim g} [x]  | \leq ( \max_{x\in \Xcal} |x| )  \| f - g \|_1 \lesssim ( \max_{x\in \Xcal} |x| )  \mathbb H(f \Mid g )\lesssim  (\max_{x\in \Xcal} |x| ) \sqrt{ D_\triangle(f \Mid g)    }$. Such an approach would have to pay the maximum range $ \max_{x\in \Xcal} |x|$ and thus can not leverage the variance $\var_{f}$. In the next sections, we show this lemma plays an important role in achieving horizon-free and second-order bounds.

\vspace{-0.26cm}
\section{Online Setting}
\label{sec:online}

\begin{algorithm}[t]%
\caption{Optimistic Model-based RL (O-MBRL)}
\label{alg:mleonline}
\resizebox{0.93\columnwidth}{!}{
\begin{minipage}{\columnwidth}
\begin{algorithmic}[1]
    \STATE\textbf{Input:} model class $\Pcal$, confidence parameter $\delta\in(0,1)$, threshold $\beta$.
    \STATE Initialize $\pi^0$, initialize dataset $\Dcal = \emptyset$.
    \FOR{$k = 0 \to K-1$}
    	\STATE Collect a trajectory $\tau = \{ s_0,a_0, \cdots, s_{H-1},a_{H-1} \}$ from $\pi^k$, split it into tuples of $\{s,a,s'\}$ and add to $\Dcal$.
	\STATE Construct a version space $\widehat \Pcal^k$:  
		\begin{align*}
	        		\widehat\Pcal^k = \braces{ P\in\Pcal: \sum_{s,a,s' \in\Dcal} \log P(s_i'|s_i,a_i)\geq \max_{\Tilde{P}\in\Pcal}\sum_{s,a,s' \in\Dcal}\log \Tilde P (s_i'|s_i,a_i)-\beta}.
    		\end{align*}\label{online conf}
	\STATE Set $(\pi^k,\widehat{P}^k)\leftarrow \argmax_{\pi\in\Pi, P\in\widehat\Pcal^k} V_{0; P}^\pi(s_0)$.\label{pessimistic policy}
    \ENDFOR
    \end{algorithmic}
        \end{minipage}}
\end{algorithm}
\vspace{-0.16cm}
In this section, we study the online setting. We present the optimistic model-based RL algorithm (O-MBRL) in \pref{alg:mleonline}. The algorithm starts from scratch, and iteratively maintains a version space $\widehat \Pcal^k$ of the model class using the historical data collected so far.  Again the version space is designed such that for all $k\in [0,K-1]$, we have $P^\star \in \widehat \Pcal_k$ with high probability. The policy $\pi^k$ in this case is computed via the optimism principle, i.e., it selects $\pi^k$ and $\widehat P^k$ such that $V^{\pi^k}_{\widehat P^k} \geq V^{\pi^\star}$.

Note that the algorithm design in \pref{alg:mleonline} is not new and in fact is quite standard in the model-based RL literature. For instance, \cite{sun2019model} presented a similar style of algorithm except that they use a min-max GAN style objective for learning models. \cite{zhan2022pac} used MLE oracle with optimism planning for Partially observable systems such as Predictive State Representations (PSRs), and \cite{liu2023optimistic} used them for both partially and fully observable systems.  However, their analyses do not give horizon-free and instance-dependent bounds. We show that under the structural condition that captures nonlinear function class with small eluder dimensions, \pref{alg:mleonline} achieves horizon-free and second-order bounds. Besides, since second-order regret bound implies first-order bound \citep{wang2024more}, our result immediately implies a first-order bound as well.

We first introduce the $\ell_p$ Eluder dimension as follows. 
    \begin{definition}[$\ell_p$ Eluder Dimension]
        $DE_p(\Psi, \Xcal, \epsilon)$  is the eluder dimension for $\Xcal$ with function class $\Psi$, when the longest $\epsilon$-independent sequence $x^1,\dots, x^L\subseteq \Xcal$ enjoys the length less than $DE_p(\Psi, \Xcal, \epsilon)$, i.e.,  there exists $g \in \Psi$ such that for all $t\in[L]$, $\sum_{l=1}^{t-1} |g(x^l)|^p \leq \epsilon^p$ and $|g(x^t)|>\epsilon$. 
    \end{definition}


    We work with the $\ell_1$ Eluder dimension $DE_1(\Psi, \Scal \times \Acal, \epsilon)$ with the function class $\Psi$ specified as:
\begin{align*}
     \Psi=\{(s,a)\mapsto \mathbb H^2(P^\star(s,a)\Mid P(s,a)):P\in\Pcal\}\,.
\end{align*}
\begin{remark}
        The $\ell_1$ Eluder dimension has been used in previous works such as \cite{liu2022partially}. We have the following corollary to demonstrate that the $\ell_1$ dimension generalizes the original $\ell_2$ dimension of \citet{russo2013eluder}, it can capture tabular, linear, and generalized linear models.

\begin{lemma}[Proposition 19 in \citep{liu2022partially}]
    \label{lem:eluder-one-generalizes-two}
For any $\Psi,\Xcal$, $\epsilon>0$, $\DE_1(\Psi,\Xcal,\epsilon)\leq\DE_2(\Psi,\Xcal,\epsilon)$.
\end{lemma}
\end{remark}

We are ready to present our main theorem for the online RL setting.
 \begin{theorem}[Main theorem for online setting] \label{thm:online_theorem} 
For any $\delta\in (0,1)$, let $\beta=4\log\left(\frac{K\left|\Pcal\right|}{\delta}\right)$, with probability at least $1-\delta$, \pref{alg:mleonline} achieves the following regret bound:
\begin{small}
    \begin{align}
\sum_{k=0}^{K-1} (V^{\pi^\star} -  V^{\pi^k}) &\leq O\Big(\sqrt{\sum_{k=0}^{K-1}\var_{\pi^k}\cdot\text{DE}_1(\Psi,\Scal \times \Acal,1/KH)\cdot\log(KH\left|\Pcal\right|/\delta)\log(KH)}\notag\\
&\quad+\text{DE}_1(\Psi,\Scal \times \Acal,1/KH)\cdot\log(KH\left|\Pcal\right|/\delta)\log(KH)\Big)\,.
\end{align}
\end{small}
\end{theorem}

The above theorem indicates the standard and simple O-MBRL algorithm is already enough to achieve horizon-free and second-order regret bounds: our bound does not have explicit polynomial dependences on horizon $H$, the leading term scales with $\sqrt{ \sum_{k} \var_{\pi^k} }$ instead of the typical $\sqrt{K}$.

We have the following result about the first-order regret bound.

\begin{corollary}[Horizon-free and First-order regret bound] Let $\beta=4\log\left(\frac{K\left|\Pcal\right|}{\delta}\right)$, with probability at least $1-\delta$, \pref{alg:mleonline} achieves the following regret bound:
\begin{small}
    \begin{align*}
\sum_{k=0}^{K-1} V^{\pi^\star} - V^{\pi^k} &\leq O\Big(\sqrt{KV^{\pi^\star}\cdot\text{DE}_1(\Psi,\Scal \times \Acal,1/KH)\cdot\log(KH\left|\Pcal\right|/\delta)\log(KH)}\notag\\
&\quad+\text{DE}_1(\Psi,\Scal \times \Acal,1/KH)\cdot\log(KH\left|\Pcal\right|/\delta)\log(KH)\Big)\,.
\end{align*}
\end{small}
\end{corollary}
\begin{proof}
    Note that $\var_{\pi} \leq V^{\pi} \leq V^{\pi^\star}$ where the first inequality is because the trajectory-wise reward is bounded in $[0,1]$. Therefore, combining with \pref{thm:online_theorem}, we directly obtain the first-order result. 
\end{proof}
Note that the above bound scales with respect to $\sqrt{K V^{\pi^\star}}$ instead of just $\sqrt{K}$. Since $V^{\pi^\star} \leq 1$, this bound improves the worst-case regret bound when the optimal policy has total reward less than one.\footnote{Typically a first-order regret bound makes more sense in the cost minimization setting instead of reward maximization setting. We believe that our results are transferable to the cost-minimization setting.}

\paragraph{Faster rates for deterministic transitions.} When the underlying MDP has deterministic transitions, we can achieve a smaller regret bound that only depends on the number of episodes logarithmically. 
\begin{corollary}[$\log K$ regret bound with deterministic transitions] \label{corr:online_coro_faster}
When the transition dynamics of the MDP are deterministic, setting $\beta=4\log\left(\frac{K\left|\Pcal\right|}{\delta}\right)$, w.p. at least $1-\delta$, \pref{alg:mleonline} achieves:
\begin{align*}
\sum_{k=0}^{K-1} V^{\pi^\star} - V^{\pi^k} \leq O\left( \text{DE}_1(\Psi,\Scal \times \Acal,1/KH)\cdot\log(KH\left|\Pcal\right|/\delta)\log(KH)\right).
\end{align*}
\end{corollary}



\paragraph{Extension to infinite class $\Pcal$.} For infinite model class $\Pcal$, we have a similar result. First, we define the bracketing number of an infinite model class as follows.

\begin{definition}[Bracketing Number \citep{geer2000empirical}]\label{def: bracketing number}
Let $\Gcal$ be a set of functions mapping $\Xcal\to\RR$.
Given two functions $l,u$ such that $l(x)\leq u(x)$ for all $x\in\Xcal$, the bracket $[l,u]$ is the set of functions $g\in\Gcal$ such that $l(x)\leq g(x)\leq u(x)$ for all $x\in\Xcal$.
We call $[l,u]$ an $\epsilon$-bracket if $\norm{u-l}\leq\epsilon$.
Then, the $\epsilon$-bracketing number of $\Gcal$ with respect to $\norm{\cdot}$, denoted by $\Ncal_{[]}(\epsilon,\Gcal,\norm{\cdot})$ is the minimum number of $\epsilon$-brackets needed to cover $\Gcal$.
\end{definition}

We use the bracketing number of $\mathcal{P}$ to denote the complexity of the model class, similar to $|\mathcal{P}|$ in the finite class case. Next, we propose a corollary to characterize the regret with an infinite model class. 
\begin{corollary}[Regret bound for \pref{alg:mleonline} with infinite model class $\Pcal$]\label{corr:online_coro_infinite} 
When the model class $\Pcal$ is infinite, let $\beta=7\log(K\Ncal_{[]}((KH|\Scal|)^{-1},\Pcal,\|\cdot\|_\infty)/\delta)$, with probability at least $1-\delta$, \pref{alg:mleonline} achieves the following regret bound:
\begin{small}
  \begin{align*}
&\sum_{k=0}^{K-1} V^{\pi^\star} - V^{\pi^k}  
\leq O\Bigg(  \text{DE}_1(\Psi,\Scal \times \Acal,\frac{1}{KH}) \log(\frac{KH\Ncal_{[]}((KH|\Scal|)^{-1},\Pcal,\|\cdot\|_\infty)}{\delta})\log(KH)\notag\\
&\quad+\sqrt{  \sum_{k=0}^{K-1}\var_{\pi^k} \cdot\text{DE}_1(\Psi,\Scal \times \Acal,\frac{1}{KH}) \log(\frac{KH\Ncal_{[]}((KH|\Scal|)^{-1},\Pcal,\|\cdot\|_\infty)}{\delta})\log(KH)}\Bigg)\,,
\end{align*}  
\end{small}

where $\Ncal_{[]}((KH|\Scal|)^{-1},\Pcal,\|\cdot\|_\infty)$ is the bracketing number defined in \pref{def: bracketing number}.
\end{corollary}

A specific example of the infinite model class is the tabular MDP, where $\mathcal{P}$ is the collection of all the conditional distributions over $\mathcal{S} \times \mathcal{A} \rightarrow \Delta(\mathcal{S})$. By \pref{corr:online_coro_infinite}, we also have a new regret bound for MBRL under the tabular MDP setting, which is nearly horizon-free and second-order. 

\begin{example}[Tabular MDPs] 
When specializing to tabular MDPs, use the fact that tabular MDP has $\ell_2$ Eluder dimension being at most $|\Scal| |\Acal|$ (Section D.1 in \cite{russo2013eluder}), $\ell_1$ dimension is upper bounded by $\ell_2$ dimension (\pref{lem:eluder-one-generalizes-two}), and use the standard $\epsilon$-net argument to show that $\Ncal_{[]}(\epsilon,\Pcal,\|\cdot\|_\infty)$ is upper-bounded by $(c/\epsilon)^{|\Scal|^2|\Acal|}$ (e.g., see  \cite{uehara2021pessimistic}), we can show that \pref{alg:mleonline} achieves the following regret bound for tabular MDP: with probability at least $1-\delta$, 
\begin{small}
        \begin{align*}
\sum_k V^{\pi^\star} - V^{\pi^k} &\leq O\Big(|\Scal|^{1.5}|\Acal|\sqrt{\sum_k\var_{\pi^k}\cdot\log(\frac{KH|\Scal|}{\delta})\log(KH)}+|\Scal|^3|\Acal|^2\log(\frac{KH|\Scal|}{\delta})\log(KH) \Big)\,.
\end{align*}
\end{small}
\end{example}


In summary, we have shown that a simple MLE-based MBRL algorithm is enough to achieve nearly horizon-free and second-order regret bounds under non-linear function approximation. 

\subsection{Proof Sketch of \pref{thm:online_theorem}}

Now we are ready to provide a proof sketch of \pref{thm:online_theorem} with the full proof deferred to \pref{app:online}. For ease of presentation, we use $\dimRL$ to denote $\text{DE}_1(\Psi,\Scal \times \Acal,1/KH)$, and ignore some $\log$ terms.

Overall, our analysis follows the general framework of optimism in the face of uncertainty, but with (1) careful analysis in leveraging the MLE generalization bound and (2) more refined proof in the training-to-testing distribution transfer via Eluder dimension. 

By standard MLE analysis, we can show w.p. $1-\delta$, for all $k\in[K-1]$, we have $P^\star \in \widehat\Pcal^k$, and
\begin{align}
     \sum_{i=0}^{k-1}\sum_{h=0}^{H-1}\mathbb H^2(P^\star(s_h^i,a_h^i)||\widehat P^k(s_h^i,a_h^i))\leq O(\log(K\left|\Pcal\right|/\delta))\,.\label{eqn: proof sketch online mle generalization}
\end{align}

From here, trivially applying training-to-testing distribution transfer via the Eluder dimension as previous works (e.g., \cite{wang2024more}) would cause poly-dependence on $H$. With new techniques detailed in \pref{app: eluder new}, which is one of our technical contributions and may be of independent interest, we can get:
there exists a set $\Kcal \subseteq [K-1]$ such that $|\Kcal| \leq O(\dimRL\log(K|\Pcal|/\delta))$, and
\begin{align}
    \sum_{k \in [K-1]\setminus \Kcal}\sum_h \mathbb H^2\Big(P^\star(s_h^k,a_h^k)\Mid \widehat P^k\big(s_h^k,a_h^k)\Big)\leq O (\dimRL\cdot\log(K\left|\Pcal\right|/\delta)\log(KH))\,.\label{eqn: proof sketch online mle eluder}
\end{align}

Recall that $(\pi^k,\widehat{P}^k)\leftarrow \argmax_{\pi\in\Pi, P\in\widehat\Pcal^k} V_{0; P}^\pi(s_0)$, with the above realization guarantee $P^\star \in \widehat\Pcal^k$, we can get the following optimism guarantee: $V^\star_{0;P^\star}\leq \max_{\pi\in\Pi,P\in\widehat \Pcal^k} V^\pi_{0;P}=V^{\pi^k}_{0;\widehat P^k}$.

At this stage, one straight-forward way to proceed is to use the standard simulation lemma (\pref{lem:simulation}): 
\begin{align}
    \sum_{k=0}^{K-1}V^{\pi^k}_{0;\widehat P^k}-V^{\pi^k}_{0;P^\star}&\leq \sum_{k=0}^{K-1}\sum_{h=0}^{H-1} \EE_{s,a\sim d^{\pi^k}_h} \left[ \left|\EE_{s'\sim P^\star(s,a)} V^{\pi^k}_{h+1;\widehat P^k}(s') -\EE_{s'\sim \widehat P^k(s,a)} V^{\pi^k}_{h+1;\widehat P^k}(s')    \right|\right]. 
\end{align}

However, from here, if we naively bound each term on the RHS via $\EE_{s,a\sim d^{\pi^k}_h} \|  P^\star(s,a) - \widehat P(s,a)     \|_1$, which is what previous works such as \cite{uehara2021pessimistic} did exactly, we would end up paying a linear horizon dependence $H$ due to the summation over $H$ on the RHS the above expression. Given the mean-to-variance lemma (\pref{lem: mean to variance}), we may consider using it to bound the difference between two means $\EE_{s'\sim P^\star(s,a)} V^{\pi^k}_{h+1;\widehat P^k}(s') -\EE_{s'\sim \widehat P^k(s,a)} V^{\pi^k}_{h+1;\widehat P^k}(s')$. This still can not work if we start from here, because we would eventually get \\$\sum_k\sum_h\EE_{s,a\sim d^{\pi^k}_h}[\mathbb H^2(P^\star(s,a)||\widehat P^k(s,a))]$ terms, which can not be further upper bounded easily with the MLE generalization guarantee.

To achieve horizon-free and second-order bounds, we need a novel and more careful analysis. 

First, we carefully decompose and upper bound the regret in $\Tilde{\Kcal}:=[K-1]\setminus \Kcal$ w.h.p. as follows using Bernstain's inequality (for regret in $\Kcal$ we simply upper bound it by $|\Kcal|$)
\begin{small}
    \begin{align}
    &\sum_{k\in\Tilde{\Kcal}}\left(V^{\pi^k}_{0;\widehat P^k}(s_h^k)-\sum_{h=0}^{H-1} r(s_h^k, a_h^k)\right)+\sum_{k\in\Tilde{\Kcal}}\left(\sum_{h=0}^{H-1} r(s_h^k, a_h^k)-V^{\pi^k}_{0;P^\star} \right)\lesssim \sqrt{\sum_{k\in\Tilde{\Kcal}}\sum_h \big(\VV_{P^\star} V_{h+1; \widehat P^k}^{\pi^k}\big)(s_h^k,a_h^k)}\notag\\
    &+ \sum_{k\in\Tilde{\Kcal}}\sum_h  \left|\EE_{s'\sim \widehat P^k(s_h^k, a_h^k)}V^{\pi^k}_{h+1;\widehat P^k}(s') - \EE_{s'\sim P^*(s_h^k, a_h^k)}V^{\pi^k}_{{h+1};\widehat P^k}(s')\right|+\sqrt{\sum_{k}\var_{\pi^k}\log(1/\delta)}\,.\label{eqn:proof sketch online decompose}
\end{align}

\end{small}



Then, we bound the difference of two means $\EE_{s'\sim \widehat P^k(s_h^k, a_h^k)}V^{\pi^k}_{h+1;\widehat P^k}(s') - \EE_{s'\sim P^*(s_h^k, a_h^k)}V^{\pi^k}_{{h+1};\widehat P^k}(s')$ using variances and the triangle discrimination (see \pref{lem: mean to variance} for more details), together with the fact that $D_\triangle \leq 4 \mathbb H^2$, and information processing inequality on the squared Hellinger distance, we have
\begin{small}
    \begin{align*}
&\lvert\EE_{s'\sim \widehat P^k(s_h^k, a_h^k)}V^{\pi^k}_{h+1;\widehat P^k}(s') - \EE_{s'\sim P^*(s_h^k, a_h^k)}V^{\pi^k}_{{h+1};\widehat P^k}(s')\rvert\\
&\quad \leq O\Big(\sqrt{\big(\VV_{P^\star} V_{h+1; \widehat P^k}^{\pi^k}\big)(s_h^k, a_h^k)D_\triangle\Big(V_{h+1; \widehat P^k}^{\pi^k}\big(s'\sim P^\star(s_h^k, a_h^k)\big)\Mid  V_{h+1; \widehat P^k}^{\pi^k}(s'\sim \widehat P^k\big(s_h^k, a_h^k)\big)\Big)}  \\
& \qquad \qquad + D_\triangle\Big(V_{h+1; \widehat P^k}^{\pi^k}\big(s'\sim P^\star(s_h^k, a_h^k)\big)\Mid  V_{h+1; \widehat P^k}^{\pi^k}(s'\sim \widehat P^k\big(s_h^k, a_h^k)\big)\Big)\Big)\\
&\leq O\Big(\sqrt{\big(\VV_{P^\star} V_{h+1; \widehat P^k}^{\pi^k}\big)(s_h^k, a_h^k)\mathbb H^2\Big(P^\star(s_h^k, a_h^k)\Mid \widehat P^k\big(s_h^k, a_h^k)\Big)}+\mathbb H^2\Big(P^\star(s_h^k, a_h^k)\Mid \widehat P^k\big(s_h^k, a_h^k)\Big)
\Big)
\end{align*} 
\end{small}
where we denote $V^{\pi^*}_{h+1;\widehat P}(s' \sim P^\star(s,a))$ as the distribution of the random variable $ V^{\pi^*}_{h+1;\widehat P}(s')$ with $s'\sim P^\star(s,a)$. This is the key lemma used by \cite{wang2024more} to show distributional RL can achieve second-order bounds. We show that this is also crucial for achieving a horizon-free bound.  

Then, summing up over $k, h$, with Cauchy-Schwartz and the MLE generalization bound via Eluder dimension in \pref{eqn: proof sketch online mle eluder}, we have
\begin{small}
    \begin{align}
    &\sum_{k\in\Tilde{\Kcal}}\sum_h  \left|\EE_{s'\sim \widehat P^k(s_h^k, a_h^k)}V^{\pi^k}_{h+1;\widehat P^k}(s') - \EE_{s'\sim P^*(s_h^k, a_h^k)}V^{\pi^k}_{{h+1};\widehat P^k}(s')\right|\leq O\Big(\sum_{k\in\Tilde{\Kcal}}\sum_h\mathbb H^2\Big(P^\star(s_h^k, a_h^k)\Mid \widehat P^k\big(s_h^k, a_h^k)\Big)\notag\\
&+\sqrt{\sum_{k\in\Tilde{\Kcal}}\sum_h\big(\VV_{P^\star} V_{h+1; \widehat P^k}^{\pi^k}\big)(s_h^k, a_h^k)\sum_{k\in\Tilde{\Kcal}}\sum_h\mathbb H^2\Big(P^\star(s_h^k, a_h^k)\Mid \widehat P^k\big(s_h^k, a_h^k)\Big)}
\Big)\notag\\
&\leq O\Big(\sqrt{\sum_{k\in\Tilde{\Kcal}}\sum_h\big(\VV_{P^\star} V_{h+1; \widehat P^k}^{\pi^k}\big)(s_h^k, a_h^k)\dimRL \log(K\left|\Pcal\right|/\delta)\log(KH)}+\dimRL \log(K\left|\Pcal\right|/\delta)\log(KH)\Big)\,.\label{eqn: proof sketch online sum mean-variance final}
\end{align}
\end{small}

Note that we have $\big(\VV_{P^\star} V_{h+1; \widehat P^k}^{\pi^k}\big)(s_h^k, a_h^k)$ depending on $\widehat P^k$. To get a second-order bound, we need to convert it to the variance under ground truth transition $P^\star$, and we want to do it without incurring any $H$ dependence. \emph{This is another key difference from \cite{wang2024more}.}

We aim to replace $\big(\VV_{P^\star} V_{h+1; \widehat P^k}^{\pi^k}\big)(s_h^k, a_h^k)$ by $\big(\VV_{P^\star} V_{h+1}^{\pi^k}\big)(s_h^k, a_h^k)$ which is the variance under $P^\star$ (recall that $V^{\pi}$ is the value function of $\pi$ under $P^\star$), and we want to control the difference $\big(\VV_{P^\star} \left( V_{h+1; \widehat P^k}^{\pi^k} - V_{h+1}^{\pi^k} \right) \big)(s_h^k, a_h^k)$.
To do so, we need to bound the $2^m$ moment of the difference $V_{h+1; \widehat P^k}^{\pi^k} - V_{h+1}^{\pi^k}$ following the strategy in \citet{zhang2021reinforcement,zhou2022computationally,zhao2023variance}. Let us define the following terms:
\begin{small}
    \begin{align}
    &A:=\sum_{k\in\tilde\Kcal}\sum_h \left[\big(\VV_{P^\star} V_{h+1; \widehat P^k}^{\pi^k}\big)(s_h^k,a_h^k)\right], C_m:=\sum_{k\in\tilde\Kcal}\sum_h\left[\big(\VV_{P^\star}( V_{h+1; \widehat P^k}^{\pi^k}-V_{h+1}^{\pi^k})^{2^m}\big)(s_h^k,a_h^k)\right], \notag\\
&B:=\sum_{k\in\tilde\Kcal}\sum_h\left[\big(\VV_{P^\star} V_{h+1}^{\pi^k}\big)(s_h^k,a_h^k)\right],
    G := \sqrt{A\cdot\dimRL \log(\frac{K\left|\Pcal\right|}{\delta})\log(KH)}+\dimRL \log(\frac{K\left|\Pcal\right|}{\delta})\log(KH)\notag\,.
\end{align}
\end{small}

With the fact $\VV_{P^\star}(a+b)\leq 2\VV_{P^\star}(a)+2\VV_{P^\star}(b)$ we have $A\leq 2B+2C_0$. For $C_m$, we prove that w.h.p. it has the recursive form $C_m\lesssim 2^mG+\sqrt{\log(1/\delta)C_{m+1}}+\log(1/\delta)$, during which process we also leverage the above \pref{eqn: proof sketch online sum mean-variance final} and some careful analysis (detailed in \pref{app:online}). Then, with the recursion lemma (\pref{lem:recursion bound C_m}), we can get $C_0\lesssim G$, which further gives us
\begin{small}
    \begin{align}
    A\lesssim B+\dimRL \log(\frac{K\left|\Pcal\right|}{\delta})\log(KH) +\sqrt{A\cdot\dimRL \log(\frac{K\left|\Pcal\right|}{\delta})\log(KH)} \leq O\big(B+\dimRL \log(\frac{K\left|\Pcal\right|}{\delta})\log(KH)\big)\,,\notag
\end{align} 
\end{small}

where in the last step we use the fact $x\leq 2a+b^2$ if $x\leq a+b\sqrt{x}$.
Finally, we note that $B \leq O(\sum_k \var_{\pi^k}+\log(1/\delta))$ w.h.p.. Plugging the upper bound of $A$ back into \pref{eqn: proof sketch online sum mean-variance final} and then to \pref{eqn:proof sketch online decompose}, we conclude the proof.

\section{Offline Setting}
\label{sec:offline}

For the offline setting, we directly analyze the Constrained Pessimism Policy Optimization (CPPO-LR) algorithm (\pref{alg:mleoffline}) proposed by \cite{uehara2021pessimistic}. We first explain the algorithm and then present its performance gap guarantee in finding the comparator policy $\pi^*$. 

\pref{alg:mleoffline} splits the offline trajectory data that contains $K$ trajectories into a dataset of $(s,a,s')$ tuples (note that in total we have $n:=KH$ many tuples)  which is used to perform maximum likelihood estimation $\max_{ \Tilde P\in \Pcal } \sum_{i=1}^n \log \Tilde P(s'_i | s_i, a_i)$. It then builds a version space $\widehat \Pcal$ which contains models $P\in \Pcal$ whose log data likelihood is not  below by too much than that of the MLE estimator. The threshold for the version space is constructed so that with high probability, $P^\star \in \widehat \Pcal$. Once we build a version space, we perform pessimistic planning to compute $\widehat \pi$.

\begin{algorithm}[t]%
\caption{(\cite{uehara2021pessimistic}) Constrained Pessimistic Policy Optimization with Likelihood-Ratio based constraints (CPPO-LR)}
\label{alg:mleoffline}
\resizebox{0.93\columnwidth}{!}{
\begin{minipage}{\columnwidth}
\begin{algorithmic}[1]
    \STATE\textbf{Input:} dataset $\Dcal = \{s,a,s'\}$, model class $\Pcal$, policy class $\Pi$, confidence parameter $\delta\in(0,1)$, threshold $\beta$.
    \STATE Calculate the confidence set based on the offline dataset:
    \begin{align*}
        \widehat\Pcal = \braces{ P\in\Pcal: \sum_{i=1}^n \log P(s_i'|s_i,a_i)\geq \max_{\Tilde{P}\in\Pcal}\sum_{i=1}^n \log \Tilde P (s_i'|s_i,a_i)-\beta}.
    \end{align*} \label{offline conf}
    \STATE \textbf{Output:} $\hat\pi\leftarrow \argmax_{\pi\in\Pi}\min_{P\in\widehat\Pcal} V_{0; P}^\pi(s_0)$.
    \end{algorithmic}
    \end{minipage}}
\end{algorithm}

We first define the single policy coverage condition as follows. 
\begin{definition}[Single policy coverage] \label{def: coverage offline} Given any comparator policy $\pi^*$, denote the data-dependent single policy concentrability coefficient $C^{\pi^*}_{\Dcal}$ as follows:
\begin{small}
    \begin{align*}
C^{\pi^*}_{\Dcal} := \max_{h, P \in \Pcal}  \frac{ \EE_{s,a\sim d^{\pi^*}_h} \mathbb{H}^2\left( P(s,a) \Mid P^\star(s,a) \right)   }{ 1/K \sum_{k=1}^K \mathbb H^2\left(   P(s_h^k,a_h^k)  \Mid   P^\star(s_h^k,a_h^k)  \right)    }.
\end{align*} 
\end{small}
We assume w.p. at least $1-\delta$ over the randomness of the generation of $\Dcal$, we have $C^{\pi^*}_{\Dcal}\leq C^{\pi^*}$.
\end{definition} 
The existence of $C^{\pi^*}$ is certainly an assumption. We now give an example in the tabular MDP where we show that if the data is generated from some fixed behavior policy $\pi^b$ which has non-trivial probability of visiting every state-action pair, then we can show the existence of $C^{\pi^*}$.
\begin{example}[Tabular MDP with good behavior policy coverage]\label{ex: offline coverage}
    If the $K$ trajectories are collected $i.i.d.$ with a fixed behavior policy $\pi^b$, and $d^{\pi^b}_h(s,a) \geq \rho_{\min}, \forall s,a, h$ (similar to \cite{ren2021nearly}), then we have: if $K$ is large enough, i.e., $K\geq 2\log(|\Scal||\Acal|H)/\rho_{\min}^2$, w.p. at least $1-\delta$, $C^{\pi^*}_{\Dcal}\leq 2/\rho_{\min}$. 
\end{example}

Our coverage definition (\pref{def: coverage offline}) shares similar spirits as the one in \cite{ye2024corruption}. It reflects how well the state-action samples in the offline dataset $\Dcal$ cover the state-action pairs induced by the comparator policy $\pi^\star$. It is different from the coverage definition in \cite{uehara2021pessimistic} in which the denominator is {\small$\EE_{s,a\sim d^{\pi^b}_h} \mathbb{H}^2\left( P(s,a) \Mid P^\star(s,a) \right)$} where $\pi^b$ is the fixed behavior policy used to collect $\Dcal$. This definition does not apply in our setting since $\Dcal$ is not necessarily generated by some underlying fixed behavior policy. On the other hand, our horizon-free result does not hold in the setting of \cite{uehara2021pessimistic} where $\Dcal$ is collected with a fixed behavior policy $\pi^b$ with the concentrability coefficient defined in their way. We leave the derivation of horizon-free results in the setting from \cite{uehara2021pessimistic}  as a future work.

Now we are ready to present the main theorem of \pref{alg:mleoffline}, which provides a tighter performance gap than that by \cite{uehara2021pessimistic}.

\begin{theorem}[Performance gap of \pref{alg:mleoffline}]\label{thm:mleoffline}
For any $\delta\in(0,1)$, let $\beta=4\log(|\Pcal|/\delta)$, w.p. at least $1-\delta$, \pref{alg:mleoffline} learns a policy $\widehat\pi$ that enjoys the following performance gap with respect to any comparator policy $\pi^*$:
\begin{equation*}
    V^{\pi^*}-V^{\widehat\pi} \leq O\left(\sqrt{{C^{\pi^*} \var_{\pi^*}\log(|\Pcal|/\delta)}/{K} } + {C^{\pi^*}\log(|\Pcal|/\delta)}/{K}\right)\,.
\end{equation*}
\end{theorem}

Comparing to the theorem (Theorem 2) of CPPO-LR from \cite{uehara2021pessimistic}, our bound has two improvements. First, our bound is horizon-free (not even any $\log (H)$ dependence), while the bound in  \cite{uehara2021pessimistic} has $\text{poly}(H)$ dependence. Second, our bound scales with $\var_{\pi^*} \in [0,1]$, which can be small when $\var_{\pi^*} \ll 1$. For deterministic system and policy $\pi^*$, we have $\var_{\pi^*} = 0$ which means the sample complexity now scales at a faster rate $C^{\pi^*} / K$. The proof is in \pref{app:offline}.

We show that the same algorithm can achieve $1/K$ rate when $P^\star$ is deterministic (but rewards could be random, and the algorithm does not need to know the condition that $P^\star$ is deterministic).

\begin{corollary}[${C^{\pi^*}}/{K}$ performance gap of \pref{alg:mleoffline} with deterministic transitions] \label{corr:coro_faster}
When the ground truth transition $P^\star$ of the MDP is deterministic, for any $\delta\in(0,1)$, let $\beta=4\log(|\Pcal|/\delta)$, w.p. at least $1-\delta$, \pref{alg:mleoffline} learns a policy $\widehat\pi$ that enjoys the following performance gap with respect to any comparator policy $\pi^*$:
\begin{equation*}
    V^{\pi^*}-V^{\widehat\pi} \leq O\left({C^{\pi^*}\log(|\Pcal|/\delta)}/{K}\right)\,.
\end{equation*} 
\end{corollary}



For infinite model class $\Pcal$, we have a similar result in the following corollary. 


\begin{corollary}[Performance gap of \pref{alg:mleoffline} with infinite model class $\Pcal$] \label{corr:offline_coro_infinite}
When the model class $\Pcal$ is infinite, for any $\delta\in(0,1)$, let $\beta=7\log(\Ncal_{[]}((KH|\Scal|)^{-1},\Pcal,\|\cdot\|_\infty)/\delta)$, w.p. at least $1-\delta$, \pref{alg:mleoffline} learns a policy $\widehat\pi$ that enjoys the following PAC bound w.r.t. any comparator policy $\pi^*$:
\begin{footnotesize}
           \begin{align*}
    V^{\pi^*}-V^{\widehat\pi}\leq O\left(\sqrt{\frac{C^{\pi^*} \var_{\pi^*}\log(\Ncal_{[]}((KH|\Scal|)^{-1},\Pcal,\|\cdot\|_\infty)/\delta)}{K} } + \frac{C^{\pi^*}\log(\Ncal_{[]}((KH|\Scal|)^{-1},\Pcal,\|\cdot\|_\infty)/\delta)}{K}\right),
\end{align*} 
\end{footnotesize}
 where $\Ncal_{.[]}((KH|\Scal|)^{-1},\Pcal,\|\cdot\|_\infty)$ is the bracketing number defined in \pref{def: bracketing number}.
\end{corollary}

Our next example gives the explicit performance gap bound for tabular MDPs. 

\begin{example}[Tabular MDPs] 
For tabular MDPs, we have $\Ncal_{[]}(\epsilon,\Pcal,\|\cdot\|_\infty)$ upper-bounded by $(c/\epsilon)^{|\Scal|^2|\Acal|}$ (e.g., see \cite{uehara2021pessimistic}). Then with probability at least $1-\delta$, let $\beta=7\log(\Ncal_{[]}((KH|\Scal|)^{-1},\Pcal,\|\cdot\|_\infty)/\delta)$, \pref{alg:mleoffline} learns a policy $\widehat\pi$ satisfying the following performance gap with respect to any comparator policy $\pi^*$:
\begin{equation*}
    V^{\pi^*}-V^{\widehat\pi} \leq O\left(|\Scal|\sqrt{{|\Acal| C^{\pi^*} \var_{\pi^*}\log(KH|\Scal|/\delta)}/{K} } + {|\Scal|^2|\Acal|C^{\pi^*}\log(KH|\Scal|/\delta)}/{K}\right),
\end{equation*}
\end{example}

The closest result to us is from \cite{ren2021nearly}, which analyzes the MBRL for tabular MDPs and obtains a performance gap $\tilde O(\sqrt{\frac{1}{Kd_m}} + \frac{|\mathcal{S}|}{Kd_m})$, where $d_m$ is the minimum visiting probability for the behavior policy to visit each state and action. Note that their result is not instance-dependent, which makes their gap only $\tilde O(1/\sqrt{K})$ even when the environment is deterministic and $\pi^*$ is deterministic. In a sharp contrast, our analysis shows a better $\tilde O(1/K)$ gap under the deterministic environment. Our result would still have the $\log H$ dependence, and we leave getting rid of this logarithmic dependence on the horizon $H$ as an open problem.

\section{Conclusion}

In this work, we presented a minimalist approach for achieving horizon-free and second-order regret bounds in RL: simply train transition models via Maximum Likelihood Estimation followed by optimistic or pessimistic planning, depending on whether we operate in the online or offline learning mode.  Our horizon-free bounds for general function approximation look quite similar to the bounds in Contextual bandits, indicating that the need for long-horizon planning does not make RL harder than CB from a statistical perspective. 


Our work has some limitations. First, when extending our result to continuous function class, we pay $\ln(H)$. This $\ln(H)$ is coming from a naive application of the $\epsilon$-net/bracket argument to the generalization bounds of MLE. We conjecture that this $\ln(H)$ can be elimiated by using a more careful analysis that uses techiniques such as peeling/chaining \citep{dudley1978central,zhang2006varepsilon}. We leave this as an important future direction. 
Second, while our model-based framework is quite general, it cannot capture problems that need to be solved via model-free approaches such as linear MDPs \citep{jin2020provably}. An interesting future work is to see if we can develop the corresponding model-free approaches that can achieve horizon-free and instance-dependent bounds for RL with general function approximation. Finally, the algorithms studied in this work are not computationally tractable. This is due to the need of performing optimism/pessimism planning for exploration. Deriving computationally tractacle RL algorithms for the rich function approximation setting is a long-standing question. 


\bibliography{reference}

\begin{thebibliography}{53}
\providecommand{\natexlab}[1]{#1}
\providecommand{\url}[1]{\texttt{#1}}
\expandafter\ifx\csname urlstyle\endcsname\relax
  \providecommand{\doi}[1]{doi: #1}\else
  \providecommand{\doi}{doi: \begingroup \urlstyle{rm}\Url}\fi

\bibitem[Aboaf et~al.(1989)Aboaf, Drucker, and Atkeson]{aboaf1989task}
Eric~W Aboaf, Steven~Mark Drucker, and Christopher~G Atkeson.
\newblock Task-level robot learning: Juggling a tennis ball more accurately.
\newblock In \emph{Proceedings, 1989 International Conference on Robotics and Automation}, pages 1290--1295. IEEE, 1989.

\bibitem[Agarwal et~al.(2019)Agarwal, Jiang, Kakade, and Sun]{agarwal2019reinforcement}
Alekh Agarwal, Nan Jiang, Sham~M Kakade, and Wen Sun.
\newblock Reinforcement learning: Theory and algorithms.
\newblock 2019.

\bibitem[Agarwal et~al.(2020)Agarwal, Kakade, Krishnamurthy, and Sun]{agarwal2020flambe}
Alekh Agarwal, Sham Kakade, Akshay Krishnamurthy, and Wen Sun.
\newblock Flambe: Structural complexity and representation learning of low rank mdps.
\newblock \emph{Advances in neural information processing systems}, 33:\penalty0 20095--20107, 2020.

\bibitem[Cesa-Bianchi et~al.(2007)Cesa-Bianchi, Mansour, and Stoltz]{cesa2007improved}
Nicolo Cesa-Bianchi, Yishay Mansour, and Gilles Stoltz.
\newblock Improved second-order bounds for prediction with expert advice.
\newblock \emph{Machine Learning}, 66:\penalty0 321--352, 2007.

\bibitem[Chua et~al.(2018)Chua, Calandra, McAllister, and Levine]{chua2018deep}
Kurtland Chua, Roberto Calandra, Rowan McAllister, and Sergey Levine.
\newblock Deep reinforcement learning in a handful of trials using probabilistic dynamics models.
\newblock \emph{Advances in neural information processing systems}, 31, 2018.

\bibitem[Deisenroth et~al.(2011)Deisenroth, Rasmussen, and Fox]{deisenroth2011learning}
Marc Deisenroth, Carl Rasmussen, and Dieter Fox.
\newblock Learning to control a low-cost manipulator using data-efficient reinforcement learning.
\newblock \emph{Robotics: Science and Systems VII}, 2011.

\bibitem[Di et~al.(2023)Di, He, Zhou, and Gu]{di2023nearly}
Qiwei Di, Jiafan He, Dongruo Zhou, and Quanquan Gu.
\newblock Nearly minimax optimal regret for learning linear mixture stochastic shortest path.
\newblock In \emph{International Conference on Machine Learning}, pages 7837--7864. PMLR, 2023.

\bibitem[Dudley(1978)]{dudley1978central}
Richard~M Dudley.
\newblock Central limit theorems for empirical measures.
\newblock \emph{The Annals of Probability}, pages 899--929, 1978.

\bibitem[Foster and Krishnamurthy(2021)]{foster2021efficient}
Dylan~J Foster and Akshay Krishnamurthy.
\newblock Efficient first-order contextual bandits: Prediction, allocation, and triangular discrimination.
\newblock \emph{Advances in Neural Information Processing Systems}, 34:\penalty0 18907--18919, 2021.

\bibitem[Foster et~al.(2021)Foster, Kakade, Qian, and Rakhlin]{foster2021statistical}
Dylan~J Foster, Sham~M Kakade, Jian Qian, and Alexander Rakhlin.
\newblock The statistical complexity of interactive decision making.
\newblock \emph{arXiv preprint arXiv:2112.13487}, 2021.

\bibitem[Foster et~al.(2024)Foster, Block, and Misra]{foster2024behavior}
Dylan~J Foster, Adam Block, and Dipendra Misra.
\newblock Is behavior cloning all you need? understanding horizon in imitation learning.
\newblock \emph{arXiv preprint arXiv:2407.15007}, 2024.

\bibitem[Geer(2000)]{geer2000empirical}
Sara~A Geer.
\newblock \emph{Empirical Processes in M-estimation}, volume~6.
\newblock Cambridge university press, 2000.

\bibitem[Huang et~al.(2024)Huang, Zhong, Wang, and Yang]{huang2024horizon}
Jiayi Huang, Han Zhong, Liwei Wang, and Lin Yang.
\newblock Horizon-free and instance-dependent regret bounds for reinforcement learning with general function approximation.
\newblock In \emph{International Conference on Artificial Intelligence and Statistics}, pages 3673--3681. PMLR, 2024.

\bibitem[Ito et~al.(2020)Ito, Hirahara, Soma, and Yoshida]{ito2020tight}
Shinji Ito, Shuichi Hirahara, Tasuku Soma, and Yuichi Yoshida.
\newblock Tight first-and second-order regret bounds for adversarial linear bandits.
\newblock \emph{Advances in Neural Information Processing Systems}, 33:\penalty0 2028--2038, 2020.

\bibitem[Jin et~al.(2018)Jin, Allen-Zhu, Bubeck, and Jordan]{jin2018q}
Chi Jin, Zeyuan Allen-Zhu, Sebastien Bubeck, and Michael~I Jordan.
\newblock Is q-learning provably efficient?
\newblock \emph{Advances in neural information processing systems}, 31, 2018.

\bibitem[Jin et~al.(2020)Jin, Yang, Wang, and Jordan]{jin2020provably}
Chi Jin, Zhuoran Yang, Zhaoran Wang, and Michael~I Jordan.
\newblock Provably efficient reinforcement learning with linear function approximation.
\newblock In \emph{Conference on learning theory}, pages 2137--2143. PMLR, 2020.

\bibitem[Kaiser et~al.(2019)Kaiser, Babaeizadeh, Milos, Osinski, Campbell, Czechowski, Erhan, Finn, Kozakowski, Levine, et~al.]{kaiser2019model}
Lukasz Kaiser, Mohammad Babaeizadeh, Piotr Milos, Blazej Osinski, Roy~H Campbell, Konrad Czechowski, Dumitru Erhan, Chelsea Finn, Piotr Kozakowski, Sergey Levine, et~al.
\newblock Model-based reinforcement learning for atari.
\newblock \emph{arXiv preprint arXiv:1903.00374}, 2019.

\bibitem[Kim et~al.(2022)Kim, Yang, and Jun]{kim2022improved}
Yeoneung Kim, Insoon Yang, and Kwang-Sung Jun.
\newblock Improved regret analysis for variance-adaptive linear bandits and horizon-free linear mixture mdps.
\newblock \emph{Advances in Neural Information Processing Systems}, 35:\penalty0 1060--1072, 2022.

\bibitem[Li et~al.(2022)Li, Wang, and Yang]{li2022settling}
Yuanzhi Li, Ruosong Wang, and Lin~F Yang.
\newblock Settling the horizon-dependence of sample complexity in reinforcement learning.
\newblock In \emph{2021 IEEE 62nd Annual Symposium on Foundations of Computer Science (FOCS)}, pages 965--976. IEEE, 2022.

\bibitem[Liu et~al.(2022)Liu, Chung, Szepesv{\'a}ri, and Jin]{liu2022partially}
Qinghua Liu, Alan Chung, Csaba Szepesv{\'a}ri, and Chi Jin.
\newblock When is partially observable reinforcement learning not scary?
\newblock In \emph{Conference on Learning Theory}, pages 5175--5220. PMLR, 2022.

\bibitem[Liu et~al.(2023)Liu, Netrapalli, Szepesvari, and Jin]{liu2023optimistic}
Qinghua Liu, Praneeth Netrapalli, Csaba Szepesvari, and Chi Jin.
\newblock Optimistic mle: A generic model-based algorithm for partially observable sequential decision making.
\newblock In \emph{Proceedings of the 55th Annual ACM Symposium on Theory of Computing}, pages 363--376, 2023.

\bibitem[Mania et~al.(2019)Mania, Tu, and Recht]{mania2019certainty}
Horia Mania, Stephen Tu, and Benjamin Recht.
\newblock Certainty equivalence is efficient for linear quadratic control.
\newblock \emph{Advances in Neural Information Processing Systems}, 32, 2019.

\bibitem[Olkhovskaya et~al.(2024)Olkhovskaya, Mayo, van Erven, Neu, and Wei]{olkhovskaya2024first}
Julia Olkhovskaya, Jack Mayo, Tim van Erven, Gergely Neu, and Chen-Yu Wei.
\newblock First-and second-order bounds for adversarial linear contextual bandits.
\newblock \emph{Advances in Neural Information Processing Systems}, 36, 2024.

\bibitem[Ren et~al.(2021)Ren, Li, Dai, Du, and Sanghavi]{ren2021nearly}
Tongzheng Ren, Jialian Li, Bo~Dai, Simon~S Du, and Sujay Sanghavi.
\newblock Nearly horizon-free offline reinforcement learning.
\newblock \emph{Advances in neural information processing systems}, 34:\penalty0 15621--15634, 2021.

\bibitem[Ross and Bagnell(2012)]{ross2012agnostic}
Stephane Ross and J~Andrew Bagnell.
\newblock Agnostic system identification for model-based reinforcement learning.
\newblock \emph{arXiv preprint arXiv:1203.1007}, 2012.

\bibitem[Russo and Van~Roy(2013)]{russo2013eluder}
Daniel Russo and Benjamin Van~Roy.
\newblock Eluder dimension and the sample complexity of optimistic exploration.
\newblock \emph{Advances in Neural Information Processing Systems}, 26, 2013.

\bibitem[Song and Sun(2021)]{song2021pc}
Yuda Song and Wen Sun.
\newblock Pc-mlp: Model-based reinforcement learning with policy cover guided exploration.
\newblock In \emph{International Conference on Machine Learning}, pages 9801--9811. PMLR, 2021.

\bibitem[Sun et~al.(2019)Sun, Jiang, Krishnamurthy, Agarwal, and Langford]{sun2019model}
Wen Sun, Nan Jiang, Akshay Krishnamurthy, Alekh Agarwal, and John Langford.
\newblock Model-based rl in contextual decision processes: Pac bounds and exponential improvements over model-free approaches.
\newblock In \emph{Conference on learning theory}, pages 2898--2933. PMLR, 2019.

\bibitem[Tarbouriech et~al.(2021)Tarbouriech, Zhou, Du, Pirotta, Valko, and Lazaric]{tarbouriech2021stochastic}
Jean Tarbouriech, Runlong Zhou, Simon~S Du, Matteo Pirotta, Michal Valko, and Alessandro Lazaric.
\newblock Stochastic shortest path: Minimax, parameter-free and towards horizon-free regret.
\newblock \emph{Advances in neural information processing systems}, 34:\penalty0 6843--6855, 2021.

\bibitem[Uehara and Sun(2021)]{uehara2021pessimistic}
Masatoshi Uehara and Wen Sun.
\newblock Pessimistic model-based offline reinforcement learning under partial coverage.
\newblock \emph{arXiv preprint arXiv:2107.06226}, 2021.

\bibitem[Uehara et~al.(2021)Uehara, Zhang, and Sun]{uehara2021representation}
Masatoshi Uehara, Xuezhou Zhang, and Wen Sun.
\newblock Representation learning for online and offline rl in low-rank mdps.
\newblock \emph{arXiv preprint arXiv:2110.04652}, 2021.

\bibitem[Venkatraman et~al.(2017)Venkatraman, Capobianco, Pinto, Hebert, Nardi, and Bagnell]{venkatraman2017improved}
Arun Venkatraman, Roberto Capobianco, Lerrel Pinto, Martial Hebert, Daniele Nardi, and J~Andrew Bagnell.
\newblock Improved learning of dynamics models for control.
\newblock In \emph{2016 International Symposium on Experimental Robotics}, pages 703--713. Springer, 2017.

\bibitem[Wagenmaker et~al.(2022)Wagenmaker, Chen, Simchowitz, Du, and Jamieson]{wagenmaker2022first}
Andrew~J Wagenmaker, Yifang Chen, Max Simchowitz, Simon Du, and Kevin Jamieson.
\newblock First-order regret in reinforcement learning with linear function approximation: A robust estimation approach.
\newblock In \emph{International Conference on Machine Learning}, pages 22384--22429. PMLR, 2022.

\bibitem[Wang et~al.(2023)Wang, Zhou, Wu, Kallus, and Sun]{wang2023benefits}
Kaiwen Wang, Kevin Zhou, Runzhe Wu, Nathan Kallus, and Wen Sun.
\newblock The benefits of being distributional: Small-loss bounds for reinforcement learning.
\newblock \emph{Advances in Neural Information Processing Systems}, 36, 2023.

\bibitem[Wang et~al.(2024)Wang, Oertell, Agarwal, Kallus, and Sun]{wang2024more}
Kaiwen Wang, Owen Oertell, Alekh Agarwal, Nathan Kallus, and Wen Sun.
\newblock More benefits of being distributional: Second-order bounds for reinforcement learning.
\newblock \emph{arXiv preprint arXiv:2402.07198}, 2024.

\bibitem[Wang et~al.(2020)Wang, Du, Yang, and Kakade]{wang2020long}
Ruosong Wang, Simon~S Du, Lin~F Yang, and Sham~M Kakade.
\newblock Is long horizon reinforcement learning more difficult than short horizon reinforcement learning?
\newblock \emph{arXiv preprint arXiv:2005.00527}, 2020.

\bibitem[Williams et~al.(2017)Williams, Wagener, Goldfain, Drews, Rehg, Boots, and Theodorou]{williams2017information}
Grady Williams, Nolan Wagener, Brian Goldfain, Paul Drews, James~M Rehg, Byron Boots, and Evangelos~A Theodorou.
\newblock Information theoretic mpc for model-based reinforcement learning.
\newblock In \emph{2017 IEEE international conference on robotics and automation (ICRA)}, pages 1714--1721. IEEE, 2017.

\bibitem[Yang et~al.(2023)Yang, Du, Ghasemipour, Tompson, Schuurmans, and Abbeel]{yang2023learning}
Mengjiao Yang, Yilun Du, Kamyar Ghasemipour, Jonathan Tompson, Dale Schuurmans, and Pieter Abbeel.
\newblock Learning interactive real-world simulators.
\newblock \emph{arXiv preprint arXiv:2310.06114}, 2023.

\bibitem[Ye et~al.(2024)Ye, Yang, Gu, and Zhang]{ye2024corruption}
Chenlu Ye, Rui Yang, Quanquan Gu, and Tong Zhang.
\newblock Corruption-robust offline reinforcement learning with general function approximation.
\newblock \emph{Advances in Neural Information Processing Systems}, 36, 2024.

\bibitem[Zanette and Brunskill(2019)]{zanette2019tighter}
Andrea Zanette and Emma Brunskill.
\newblock Tighter problem-dependent regret bounds in reinforcement learning without domain knowledge using value function bounds.
\newblock In \emph{International Conference on Machine Learning}, pages 7304--7312. PMLR, 2019.

\bibitem[Zhan et~al.(2022)Zhan, Uehara, Sun, and Lee]{zhan2022pac}
Wenhao Zhan, Masatoshi Uehara, Wen Sun, and Jason~D Lee.
\newblock Pac reinforcement learning for predictive state representations.
\newblock \emph{arXiv preprint arXiv:2207.05738}, 2022.

\bibitem[Zhang et~al.(2023)Zhang, Zhang, and Gu]{zhang2023optimal}
Junkai Zhang, Weitong Zhang, and Quanquan Gu.
\newblock Optimal horizon-free reward-free exploration for linear mixture mdps.
\newblock In \emph{International Conference on Machine Learning}, pages 41902--41930. PMLR, 2023.

\bibitem[Zhang(2006)]{zhang2006varepsilon}
Tong Zhang.
\newblock From $\varepsilon$-entropy to kl-entropy: Analysis of minimum information complexity density estimation.
\newblock \emph{The Annals of Statistics}, pages 2180--2210, 2006.

\bibitem[Zhang et~al.(2021{\natexlab{a}})Zhang, Ji, and Du]{zhang2021reinforcement}
Zihan Zhang, Xiangyang Ji, and Simon Du.
\newblock Is reinforcement learning more difficult than bandits? a near-optimal algorithm escaping the curse of horizon.
\newblock In \emph{Conference on Learning Theory}, pages 4528--4531. PMLR, 2021{\natexlab{a}}.

\bibitem[Zhang et~al.(2021{\natexlab{b}})Zhang, Yang, Ji, and Du]{zhang2021improved}
Zihan Zhang, Jiaqi Yang, Xiangyang Ji, and Simon~S Du.
\newblock Improved variance-aware confidence sets for linear bandits and linear mixture mdp.
\newblock \emph{Advances in Neural Information Processing Systems}, 34:\penalty0 4342--4355, 2021{\natexlab{b}}.

\bibitem[Zhang et~al.(2022)Zhang, Ji, and Du]{zhang2022horizon}
Zihan Zhang, Xiangyang Ji, and Simon Du.
\newblock Horizon-free reinforcement learning in polynomial time: the power of stationary policies.
\newblock In \emph{Conference on Learning Theory}, pages 3858--3904. PMLR, 2022.

\bibitem[Zhang et~al.(2024{\natexlab{a}})Zhang, Chen, Lee, and Du]{zhang2024settling}
Zihan Zhang, Yuxin Chen, Jason~D Lee, and Simon~S Du.
\newblock Settling the sample complexity of online reinforcement learning.
\newblock In \emph{The Thirty Seventh Annual Conference on Learning Theory}, pages 5213--5219. PMLR, 2024{\natexlab{a}}.

\bibitem[Zhang et~al.(2024{\natexlab{b}})Zhang, Lee, Chen, and Du]{zhang2024horizon}
Zihan Zhang, Jason~D Lee, Yuxin Chen, and Simon~S Du.
\newblock Horizon-free regret for linear markov decision processes.
\newblock \emph{arXiv preprint arXiv:2403.10738}, 2024{\natexlab{b}}.

\bibitem[Zhao et~al.(2023{\natexlab{a}})Zhao, He, and Gu]{zhao2023nearly}
Heyang Zhao, Jiafan He, and Quanquan Gu.
\newblock A nearly optimal and low-switching algorithm for reinforcement learning with general function approximation.
\newblock \emph{arXiv preprint arXiv:2311.15238}, 2023{\natexlab{a}}.

\bibitem[Zhao et~al.(2023{\natexlab{b}})Zhao, He, Zhou, Zhang, and Gu]{zhao2023variance}
Heyang Zhao, Jiafan He, Dongruo Zhou, Tong Zhang, and Quanquan Gu.
\newblock Variance-dependent regret bounds for linear bandits and reinforcement learning: Adaptivity and computational efficiency.
\newblock In \emph{The Thirty Sixth Annual Conference on Learning Theory}, pages 4977--5020. PMLR, 2023{\natexlab{b}}.

\bibitem[Zhong et~al.(2022)Zhong, Xiong, Zheng, Wang, Wang, Yang, and Zhang]{zhong2022gec}
Han Zhong, Wei Xiong, Sirui Zheng, Liwei Wang, Zhaoran Wang, Zhuoran Yang, and Tong Zhang.
\newblock Gec: A unified framework for interactive decision making in mdp, pomdp, and beyond.
\newblock \emph{arXiv preprint arXiv:2211.01962}, 2022.

\bibitem[Zhou and Gu(2022)]{zhou2022computationally}
Dongruo Zhou and Quanquan Gu.
\newblock Computationally efficient horizon-free reinforcement learning for linear mixture mdps.
\newblock \emph{Advances in neural information processing systems}, 35:\penalty0 36337--36349, 2022.

\bibitem[Zhou et~al.(2023)Zhou, Zihan, and Du]{zhou2023sharp}
Runlong Zhou, Zhang Zihan, and Simon~Shaolei Du.
\newblock Sharp variance-dependent bounds in reinforcement learning: Best of both worlds in stochastic and deterministic environments.
\newblock In \emph{International Conference on Machine Learning}, pages 42878--42914. PMLR, 2023.

\end{thebibliography}

\appendix

\appendix

\newpage
\section{Summary of Contents in the Appendix}
The Appendix is organized as follows. 

In \pref{app: eluder new}, we provide some new analyses for Eluder dimension, which we will use for proving the regret bounds for the online RL setting. 

In \pref{app: supporting lemmas}, we provide some other supporting lemmas that will be used in our proofs. 

In \pref{app: online full}, we provide the detailed proofs for the online RL setting (\pref{sec:online}). Specifically, in \pref{app:online} we give the proof of \pref{thm:online_theorem}; in \pref{app:online_coro_faster}, we show the proof of \pref{corr:online_coro_faster}; in \pref{app:online_coro_infinite}, we give the proof of \pref{corr:online_coro_infinite}. 

In \pref{app: offline full}, we provide the detailed proofs for the offline RL setting (\pref{sec:offline}). Specifically, in \pref{app:offline} we give the proof of \pref{thm:mleoffline}; in \pref{app:offline_coro_faster}, we show the proof of \pref{corr:coro_faster}; in \pref{app:offline_coro_infinite}, we give the proof of \pref{corr:offline_coro_infinite}; in \pref{app: example proof}, we show the proof of the claim in \pref{ex: offline coverage}.

\section{Analysis regarding the Eluder Dimension}\label{app: eluder new}
For simplicity, we denote $x_h^k =  (s_h^k, a_h^k)$.

First we have two technical lemma. The first lemma bounds the summation of ``self-normalization" terms by the Eluder dimension. Our result generalizes the previous result by \citet{zhao2023nearly} from the $\ell_2$-Eluder dimension to the $\ell_1$ case. 
\begin{lemma}\label{lem:normal}
Suppose for all $g \in \Psi, |g| \leq 1$ and $\lambda>1$, then we have
    \begin{align}
        &\sum_{k=1}^K\sum_{h=1}^H \min\bigg\{1,\sup_{g \in \Psi}\frac{|g(x_{h}^{k})|}{\sum_{k'=1}^{k-1}\sum_{h'=1}^H |g(x_{h'}^{k'})| + \sum_{h'=1}^{h-1} |g(x_{h'}^{k})| + \lambda}\bigg\} \notag \\
        &\leq 12\log^2(4\lambda KH)\cdot DE_1(\Psi, \Scal \times \Acal, 1/(8\lambda KH)) + \lambda^{-1}.\notag
    \end{align}
\end{lemma}
\begin{proof}[Proof of \pref{lem:normal}]
We follow the proof steps of Theorem 4.6 in \citet{zhao2023nearly}. 
    For simplicity, we use $n = KH$, $i = kH+h$ to denote the indices and denote $x_h^k$ by $x_i$. Then we need to prove
    \begin{align}
        \sum_{i=1}^n\min\bigg\{1,\sup_{g \in \Psi}\frac{|g(x_i)|}{\sum_{t=1}^{i-1}|g(x_t)| + \lambda}\bigg\} \leq 12\log^2(4\lambda n)\cdot DE_1(\Psi, \Scal \times \Acal, 1/(8\lambda n)) + \lambda^{-1}\label{eqn:help999}
    \end{align}
    Let 
        \begin{align}
        g_i = \argmax_{g \in \Psi}\frac{|g(x_i)|}{\sum_{j=1}^{i-1}|g(x_j)| + \lambda}
    \end{align}
    For any $1/(\lambda n)\leq \rho\leq 1$ and $1\leq j \leq \lceil \log(4\lambda n)\rceil$, we define
    \begin{align}
       A_\rho^j= \bigg\{i\in [n] : 2^{-j} < |g_i(x_i)|\leq 2^{-j+1}, \frac{|g_i(x_i)|}{\sum_{t=1}^{i-1}|g_i(x_t)| + \lambda}\geq \rho/2\bigg\},\ d_j:=DE_1(\Psi, \Scal \times \Acal, 2^{-j}).
    \end{align}
Next we only consider the set $A_\rho^j$ where $|A_\rho^j|>d_j$. We denote $A_\rho^j = \{a_1,\dots, a_A\}$, where $A = |A_\rho^j|$ and $\{a_i\}$ keeps the same order as $\{x_i\}$. Next we do the following constructions. We maintain $k = \lfloor (A-1)/d_j\rfloor$ number of queues $Q_1,\dots, Q_k$, all of them initialized as emptysets. We put $a_1$ into $Q_1$. For $a_i, i\geq 2$, we put $a_i$ into $Q_l$, where $Q_l$ is the first queue where $a_i$ is $2^{-j}$-independent of all elements in $Q_l$. Let $i_{\max}$ be the smallest $i$ when we can not put $a_i$ into any existing queue. 

We claim that $i_{\max}$ indeed exists, i.e., our construction will stop before we put all elements in $A^j_\rho$ into $Q_1,\dots, Q_k$. In fact, note the fact that the length of each $Q_l$ is always no more than $d_j$, which is due to the fact that any $2^{-j}$-independent sequence's length is at most $d_j$. Meanwhile, since we only have $k = \lfloor (A-1)/d_j\rfloor$, then the amount of elements in $Q_1\cup\dots\cup Q_k$ will be upper bounded by $k\cdot d_j<A$. That suggests at least one element in $A_\rho^j$ is not contained by $Q_1\cup\dots\cup Q_k$, i.e., $i_{\max}$ exists. 

By the definition of $i_{\max}$, we know that $a_{i_{\max}}$ is $2^{-j}$-dependent to each $Q_l$. Next we give a bound of $A$. First, note
\begin{align}
    \sum_{t=1}^{i_{\max}-1}|g_{i_{\max}}(x_t)| \geq \sum_{t\in Q_1\cup\dots\cup Q_k}|g_{i_{\max}}(a_t)| = \sum_{l=1}^k \sum_{t\in Q_l}|g_{i_{\max}}(a_t)| > k\cdot 2^{-j},\label{eqn: help221}
\end{align}
where the first inequality holds since $Q_l$ are the elements that appear before $a_{i_{\max}}$, the second one holds due to the following induction of Eluder dimension: since $a_{i_{\max}}$ is $2^{-j}$-dependent to $Q_l$, then we have
\begin{align}
    \forall g \in \Psi,\ \sum_{t\in Q_l}|g(a_t)| \leq 2^{-j} \Rightarrow |g(a_{i_{\max}})| \leq 2^{-j}.\label{eqn: help:222}
\end{align}
Therefore, given the fact $|g_{i_{\max}}(a_{i_{\max}})| > 2^{-j}$ (recall the definition of $A^j_\rho$), we must have $\sum_{t\in Q_l}|g_{i_{\max}}(a_t)|>2^{-j}$ as well, which suggests the second inequality of \pref{eqn: help221} holds. Second, we have
\begin{align}
    \sum_{t=1}^{i_{\max}-1}|g_{i_{\max}}(x_t)| \leq 2/\rho\cdot |g_{i_{\max}}(a_{i_{\max}})| \leq 4\cdot 2^{-j}/\rho,\label{eqn:help223}
\end{align}
where both inequalities hold due to the definition of $A^j_\rho$. Combining \pref{eqn: help221} and \pref{eqn:help223}, we have
\begin{align}
    k<4/\rho \Rightarrow A \leq 4d_j/\rho +d_j \leq 5d_j/\rho. 
\end{align}
Therefore, we have that for all $\rho, j$, $|A^j_\rho| \leq 5d_j/\rho$. 

Finally we prove \pref{eqn:help999}. $1/(\lambda n)\leq \rho\leq 1$ and $1\leq j \leq \lceil \log(4\lambda n)\rceil = J$. Denote 
   \begin{align}
       A_\rho= \bigg\{i\in [n] :\frac{|g_i(x_i)|}{\sum_{t=1}^{i-1}|g_i(x_t)| + \lambda}\geq \rho/2\bigg\}.
    \end{align}
Then it is easy to notice that $|A_\rho| = \sum_j |A^j_\rho| \leq \lceil \log(4\lambda n)\rceil\cdot 5d_J/\rho$, where we use the fact that the Eluder dimension $d_j$ is increasing. Therefore, by the standard peeling technique, we have
\begin{align}
    \sum_{i=1}^n\min\bigg\{1,\sup_{g \in \Psi}\frac{|g(x_i)|}{\sum_{t=1}^{i-1}|g(x_t)| + \lambda}\bigg\} &= \sum_{j\in [\lceil\log(\lambda n) \rceil]}\sum_{i\in A_{2^{-j}}\setminus A_{2^{-j+1}}} + \sum_{j= \lceil\log(\lambda n) \rceil }\sum_{i\notin A_{2^{-j+1}}} \notag \\
    & \leq  \sum_{j\in [\lceil\log(\lambda n) \rceil]}\sum_{i\in A_{2^{-j}}\setminus A_{2^{-j+1}}} 2^{-j-1} + n\cdot 1/(\lambda n)\notag \\
    & \leq \sum_{j\in [\lceil\log(\lambda n) \rceil]}\sum_{i\in A_{2^{-j}}} 2^{-j-1} + n\cdot 1/(\lambda n)\notag \\
    & \leq \lceil\log(\lambda n) \rceil\cdot \lceil \log(4\lambda n)\rceil\cdot 3d_J + \lambda^{-1},\notag
\end{align}
which concludes our proof.
\end{proof}

Next lemma gives a bound to bound the number of episodes where the behavior along these episodes are ``bad". Intuitively speaking, our lemma suggests we only have limited number of bad episodes, therefore won't affect the final performance of our algorithm.
   \begin{lemma}\label{lem:zhou_1}
        Given $\lambda>1$. There exists at most 
        \begin{align}
            13\log^2(4\lambda KH)\cdot DE_1(\Psi, \Scal \times \Acal, 1/(8\lambda KH))
        \end{align}
        number of $k \in [K]$ satisfying the following claim
        \begin{align}
        \sup_{g \in \Psi}\frac{\lambda + \sum_{k'=1}^{k}\sum_{h'=1}^H |g(x_{h'}^{k'})|}{\lambda + \sum_{k'=1}^{k-1}\sum_{h'=1}^H |g(x_{h'}^{k'})|} > 4.
\label{help:1}
        \end{align}
    \end{lemma}

\begin{proof}[Proof of \pref{lem:zhou_1}]

Note that 
\begin{align}
   &\sum_{k=1}^K\min\bigg\{2, \log \sup_{g \in \Psi}\frac{\lambda + \sum_{k'=1}^{k}\sum_{h'=1}^H |g(x_{h'}^{k'})|}{\lambda + \sum_{k'=1}^{k-1}\sum_{h'=1}^H |g(x_{h'}^{k'})|} \bigg\}\notag \\
   &\leq \sum_{k=1}^K\min\bigg\{2,\log \prod_{h=1}^H \sup_{g \in \Psi}\frac{\lambda + \sum_{k'=1}^{k-1}\sum_{h'=1}^H |g(x_{h'}^{k'})| + \sum_{h'=1}^{h} |g(x_{h'}^{k})|}{\lambda + \sum_{k'=1}^{k-1}\sum_{h'=1}^H |g(x_{h'}^{k'})| + \sum_{h'=1}^{h-1} |g(x_{h'}^{k})|}\bigg\}\notag \\
   & = \sum_{k=1}^K\min\bigg\{2,\sum_{h=1}^H \log\bigg(1+ \sup_{g \in \Psi}\frac{|g(x_{h}^{k})|}{\lambda + \sum_{k'=1}^{k-1}\sum_{h'=1}^H |g(x_{h'}^{k'})| + \sum_{h'=1}^{h-1} |g(x_{h'}^{k})|}\bigg)\bigg\}\notag \\
   & \leq \sum_{k=1}^K\sum_{h=1}^H \min\bigg\{2,\sup_{g \in \Psi}\frac{|g(x_{h}^{k})|}{\sum_{k'=1}^{k-1}\sum_{h'=1}^H |g(x_{h'}^{k'})| + \sum_{h'=1}^{h-1} |g(x_{h'}^{k})| + \lambda}\bigg\},\notag \\
   & \leq 2\sum_{k=1}^K\sum_{h=1}^H \min\bigg\{1,\sup_{g \in \Psi}\frac{|g(x_{h}^{k})|}{\sum_{k'=1}^{k-1}\sum_{h'=1}^H |g(x_{h'}^{k'})| + \sum_{h'=1}^{h-1} |g(x_{h'}^{k})| + \lambda}\bigg\}\notag\\
    &\leq 24\log^2(4\lambda KH)\cdot DE_1(\Psi, \Scal \times \Acal, 1/(8\lambda KH)) + 2\lambda^{-1}\notag \\
    & \leq 26\log^2(4\lambda KH)\cdot DE_1(\Psi, \Scal \times \Acal, 1/(8\lambda KH)).\label{eq:xxx}
\end{align}
where the first inequality holds since $\sup_{g}\prod f(g) \leq \prod \sup_g f(g)$, the second one holds since $\log(1+x) \leq x$, the fourth one holds due to \pref{lem:normal}. Therefore, there are at most 
\begin{align}
    26\log^2(4\lambda KH)\cdot DE_1(\Psi, \Scal \times \Acal, 1/(8\lambda KH))/2\notag
\end{align}
number of $k$ satisfying
\begin{align}
    \log \sup_{g \in \Psi}\frac{\lambda + \sum_{k'=1}^{k}\sum_{h'=1}^H |g(x_{h'}^{k'})|}{\lambda + \sum_{k'=1}^{k-1}\sum_{h'=1}^H |g(x_{h'}^{k'})|}> 2,\notag
\end{align}
which concludes the proof.

\end{proof}

We next have the following lemma, which bounds the regret by the Eluder dimension. 
\begin{lemma}[Theorem 5.3, \citealt{wang2023benefits}]\label{lem:kaiweneluder}
Let $C := \sup_{(s,a)\in\Scal\times\Acal,f\in\Psi}\abs{ f((s,a))}$ be the envelope.
For any sequences $f^{(1)},\dots,f^{(N)}\subseteq \Psi$, $(s,a)^{(1)},\dots,(s,a)^{(N)}\subseteq\Scal\times\Acal$, let $\beta$ be a constant such that for all $n\in[N]$ we have,
$
    \sum_{i=1}^{n-1}\abs{f^{(n)}((s,a)^i)} \leq \beta.
$
Then, for all $n\in[N]$, we have
\begin{equation*}
    \sum_{t=1}^n\abs{f^{(t)}((s,a)^t)}\leq \inf_{0<\epsilon\leq 1}\braces{ \text{DE}_1(\Psi,\Scal \times \Acal,\epsilon)(2C + \beta\log(C/\epsilon)) + n\epsilon }.
\end{equation*}
\end{lemma}

Given \pref{lem:zhou_1} and Lemma \pref{lem:kaiweneluder}, we are able to prove the following key lemma. 
\begin{lemma}[New Eluder Pigeon Lemma]\label{lem: new eluder pigeon lemma}
Let the event $\Ecal$ be 
\begin{align}
    \Ecal: \forall k \in [K],\ \sum_{i=1}^{k-1} \sum_{h=1}^H   \mathbb H^2(\widehat P^k(s_h^i ,a_h^i )||P^*(s_h^i ,a_h^i ))\leq \eta.
\end{align}
Then under event $\Ecal$, there exists a set $\Kcal \in [K]$ such that
\begin{itemize}[leftmargin = *]
    \item We have $|\Kcal| \leq 13\log^2(4\eta KH)\cdot DE_1(\Psi, \Scal \times \Acal, 1/(8\eta KH))$.
    \item We have
    \begin{align}
        &\sum_{k \in [K]\setminus \Kcal}\sum_{h=1}^H \mathbb H^2\Big(P^\star(s_h^k,a_h^k)\Mid  \widehat P^k\big(s_h^k,a_h^k)\Big)\notag \\
        &\leq \inf_{0<\epsilon\leq 1}\braces{ \text{DE}_1(\Psi,\Scal \times \Acal,\epsilon)(2 + 7\eta\log(1/\epsilon)) + KH\epsilon }\notag \\
        &\leq \text{DE}_1(\Psi,\Scal \times \Acal,1/KH)(2 + 7\eta\log(KH)) + 1 \,,
    \end{align}
\end{itemize}
where the function class $\Psi =\{(s,a)\mapsto \mathbb H^2(P^\star(s,a)\Mid P(s,a)):P\in\Pcal\}$.
\end{lemma}
\begin{proof}[Proof of \pref{lem: new eluder pigeon lemma}]
We interchangeably use $n = kH+h$ to denote the indices of $s_h^k, a_h^k$. We set $f^{(n)}((s,a))$ in \pref{lem:kaiweneluder} as $H^2(P^k(s,a )||P^*(s,a))$.

First, we prove that the $\beta$ in \pref{lem:kaiweneluder} can be selected as $7\eta$ under event $\Ecal$. To show that, let $\Kcal$ denote all the $k$ stated in \pref{lem:zhou_1}. Then for all $k$ such that $k+1\notin \Kcal$, $h = 2,...,H$, let $n = kH+h$, we have  
\begin{align}
    \sum_{i=0}^{n-1}\abs{f^{(n)} ((s,a)^i)} &\leq \sum_{i=0}^{kH+H}\abs{f^{(n)}((s,a)^i)} \notag \\
    & = \bigg(\lambda+\sum_{i=0}^{kH}\abs{f^{(n)}((s,a)^i)}\bigg)\cdot \frac{\sum_{i=0}^{kH+H}\abs{f^{(n)}((s,a)^i)}+\lambda}{\sum_{i=0}^{kH}\abs{f^{(n)}((s,a)^i)}+\lambda}-\lambda\notag \\
    & \leq \bigg(\lambda+\sum_{i=0}^{kH}\abs{f^{(n)}((s,a)^i)}\bigg)\cdot 4-\lambda\notag \\
    &\leq 7\eta,
\end{align}
where the second inequality holds due to \pref{lem:zhou_1}, the last one holds due to the definition of $\Ecal$. Therefore, we prove our lemma by the conclusion of \pref{lem:kaiweneluder} with $\beta = 7\eta$.  
\end{proof}

\section{Other Supporting Lemmas}\label{app: supporting lemmas}
\begin{lemma}[Simulation Lemma (\cite{agarwal2019reinforcement})]\label{lem:simulation} 
We have
\begin{align}
    V_{0;P^\star}^\pi-{V}_{0;\hat P}^\pi\leq \sum_{h=0}^{H-1} \EE_{s,a\sim d^{\pi}_h} \left[ \left|\EE_{s'\sim P^\star(s,a)} V^{\pi}_{h+1;\widehat P}(s') -\EE_{s'\sim \widehat P(s,a)} V^{\pi}_{h+1;\widehat P}(s')    \right|\right].\notag
\end{align}
    
\end{lemma}

\begin{lemma}[Change of Variance Lemma (Lemma C.5 in \cite{jin2018q})]\label{lem:variance_lemma} 
\begin{align}
    \sum_{h=0}^{H-1} \EE_{s,a\sim d^{\pi}_h}\left[\big(\VV_{P^\star} V_{h+1;P^\star}^{\pi}\big)(s,a)\right]=\var_{\pi }.\notag
\end{align} 
    
\end{lemma}

\begin{lemma}[Generalization bounds of MLE for finite model class (Theorem E.4 in \cite{wang2023benefits})]\label{lem:mle_generalization_offline}
Let $\Xcal$ be the context/feature space and $\Ycal$ be the label space, and we are given a dataset $D = \braces{ (x_i,y_i) }_{i\in[n]}$ from a martingale process:
for $i=1,2,...,n$, sample $x_i \sim \Dcal_i(x_{1:i-1},y_{1:i-1})$ and $y_i\sim p(\cdot \mid x_i)$.
Let $f^\star(x,y) = p(y\mid x)$ and we are given a realizable, \emph{i.e.}, $f^\star\in\Fcal$, function class $\Fcal:\Xcal\times\Ycal\to\Delta(\RR)$ of distributions.
Suppose $\Fcal$ is finite.
Fix any $\delta\in(0,1)$, set $\beta=\log(|\Fcal|/\delta)$ and define
\begin{align*}
    \widehat\Fcal = \braces{ f\in\Fcal: \sum_{i=1}^n\log f(x_i,y_i)\geq \max_{\widetilde f\in\Fcal}\sum_{i=1}^n \log\widetilde f(x_i,y_i) - 4\beta }.
\end{align*}
Then w.p. at least $1-\delta$, the following holds:
\begin{enumerate}
    \item[(1)] The true distribution is in the version space, \emph{i.e.}, $f^\star\in\widehat{\mathcal{F}}$.
    \item[(2)] Any function in the version space is close to the ground truth data-generating distribution, \emph{i.e.}, for all $f\in\widehat\Fcal$
    \begin{align*}
        \sum_{i=1}^n \EE_{x\sim\Dcal_i}\left[\mathbb H^2( f(x,\cdot) \Mid f^\star(x,\cdot) ) \right]\leq 22\beta.
    \end{align*}
\end{enumerate}
\end{lemma}

\begin{lemma}[Generalization bounds of MLE for infinite model class (Theorem E.5 in \cite{wang2023benefits})]\label{lem:mle_generalization infinite}
Let $\Xcal$ be the context/feature space and $\Ycal$ be the label space, and we are given a dataset $D = \braces{ (x_i,y_i) }_{i\in[n]}$ from a martingale process:
for $i=1,2,...,n$, sample $x_i \sim \Dcal_i(x_{1:i-1},y_{1:i-1})$ and $y_i\sim p(\cdot \mid x_i)$.
Let $f^\star(x,y) = p(y\mid x)$ and we are given a realizable, \emph{i.e.}, $f^\star\in\Fcal$, function class $\Fcal:\Xcal\times\Ycal\to\Delta(\RR)$ of distributions.
Suppose $\Fcal$ is finite.
Fix any $\delta\in(0,1)$, set $\beta=\log(\Ncal_{[]}((n|\Ycal|)^{-1},\Fcal,\|\cdot\|_\infty)/\delta)$ (where $\Ncal_{[]}((n|\Ycal|)^{-1},\Fcal,\|\cdot\|_\infty)$ is the bracketing number defined in \pref{def: bracketing number}) and define
\begin{align*}
    \widehat\Fcal = \braces{ f\in\Fcal: \sum_{i=1}^n\log f(x_i,y_i)\geq \max_{\widetilde f\in\Fcal}\sum_{i=1}^n \log\widetilde f(x_i,y_i) - 7\beta }.
\end{align*}
Then w.p. at least $1-\delta$, the following holds:
\begin{enumerate}
    \item[(1)] The true distribution is in the version space, \emph{i.e.}, $f^\star\in\widehat{\mathcal{F}}$.
    \item[(2)] Any function in the version space is close to the ground truth data-generating distribution, \emph{i.e.}, for all $f\in\widehat\Fcal$
    \begin{align*}
        \sum_{i=1}^n \EE_{x\sim\Dcal_i}\left[\mathbb H^2( f(x,\cdot) \Mid f^\star(x,\cdot) ) \right]\leq 28\beta.
    \end{align*}
\end{enumerate}
\end{lemma}
 \begin{lemma}[Recursion Lemma]\label{lem:recursion bound C_m}
     Let $G >0$ be a positive constant, $a<G/2$ is also a positive constant, and let $\{C_m\}_{m=0}^{N=\lceil\log_2(\frac{KH}{G})\rceil}$ be a sequence of positive real numbers satisfying:
\begin{enumerate}
    \item $C_m \leq 2^{m} G + \sqrt{a C_{m+1}}+a$ for all $m \geq 0$,
    \item $C_m \leq K H$ for all $m \geq 0$, where $K > 0$ and $H > 0$ are positive constants.
\end{enumerate}
Then, it holds that:
\[
C_0 \leq 4 G.
\]
 \end{lemma}
\begin{proof}[Proof of \pref{lem:recursion bound C_m}]
     We will prove by induction that for all $m \geq 0$,
\[
C_m \leq 2^{m+2} G.
\]
Then, for $m = 0$, this would immediately show $C_0 \leq 4 G$.

\textbf{1. The base case $m = N$}: 

Since $N=\lceil\log_2(\frac{KH}{G})\rceil$, it is obvious that $2^{N+2} G \geq K H$.  
Thus, $C_N\leq KH\leq 2^{N+2} G$, the inequality holds for $m = N$.

\textbf{2. The induction step:}

Assume that for some $m \geq 0$, for $C_{m+1}$, we have:
\[
C_{m+1} \leq 2^{m+1+2} G = 2^{m+3} G.
\]

Then, we have
\begin{align}
    C_m&\leq 2^{m} G + \sqrt{aC_{m+1}}+a\notag\\
    &\leq 2^{m}G+\sqrt{a2^{m+3}G}+a\notag\\
    &\leq 2^{m}G+\sqrt{\frac{G}{2}\cdot2^{m+3}G}+\frac{G}{2}\notag\\
    &=G\cdot(2^m+2^{m/2+1}+2^{-1})\notag\\
    &\leq G\cdot(2^m+2^{m+1}+2^m)\notag\\
    &=2^{m+2}G\,.
\end{align}
Therefore, by induction, we have 
for all $m \geq 0$,
\[
C_m \leq 2^{m+2} G.
\]
And the proof follows by setting $m=0$.
 \end{proof}
\section{Detailed Proofs for the online setting in \pref{sec:online}}\label{app: online full}
\subsection{Proof of \pref{thm:online_theorem}}
\label{app:online}
The following is the full proof of \pref{thm:online_theorem}. 

For notational simplicity, throughout this whole section, we denote 
\begin{align}
    A&:=\sum_{k\in[K-1]\setminus \Kcal}\sum_{h=0}^{H-1} \left[\big(\VV_{P^\star} V_{h+1; \widehat P^k}^{\pi^k}\big)(s_h^k,a_h^k)\right]\notag\\
    B&:=\sum_{k\in[K-1]\setminus \Kcal}\sum_{h=0}^{H-1}\left[\big(\VV_{P^\star} V_{h+1}^{\pi^k}\big)(s_h^k,a_h^k)\right],\notag\\
       C_m&:=\sum_{k\in[K-1]\setminus \Kcal}\sum_{h=0}^{H-1} \left[\big(\VV_{P^\star}( V_{h+1; \widehat P^k}^{\pi^k}-V_{h+1}^{\pi^k})^{2^m}\big)(s_h^k,a_h^k)\right]\notag\\
    G&:= \sqrt{\sum_{k\in[K-1]\setminus \Kcal}\sum_{h=0}^{H-1}  \left[\big(\VV_{P^\star} V_{h+1; \widehat P^k}^{\pi^k}\big) (s_h^k, a_h^k)\right]\cdot \text{DE}_1(\Psi,\Scal \times \Acal,1/KH)\cdot\log(K\left|\Pcal\right|/\delta)\log(KH)}\notag\\
&\quad+\text{DE}_1(\Psi,\Scal \times \Acal,1/KH)\cdot\log(K\left|\Pcal\right|/\delta)\log(KH)
\notag\\
I_h^k&:=\EE_{s'\sim P^*(s_h^k, a_h^k)}V^{\pi^k}_{h+1;\widehat P^k}(s')  - V^{\pi^k}_{h+1;\widehat P^k}(s_{h+1}^k)
\end{align}
We use $\mathbb{I}\{\cdot\}$ to denote the indicator function. We define the following events which we will later show that they happen with high probability.
\begin{align}
    \Ecal_1&:=\{\forall k\in[K-1]:P^\star\in\widehat\Pcal^k, \text{and} \sum_{i=0}^{k-1}\sum_{h=0}^{H-1}\mathbb H^2(P^\star(s_h^i,a_h^i)||\widehat P^k(s_h^i,a_h^i))\leq22\log(K\left|\Pcal\right|/\delta).\}\,,\label{eqn:event 1}\\
    \Ecal_2&:=\Big\{\sum_{k\in[K-1]\setminus \Kcal}\sum_{h=0}^{H-1} I_h^k\lesssim \sqrt{\sum_{k\in[K-1]\setminus \Kcal}\sum_{h=0}^{H-1} \big(\VV_{P^\star} V_{h+1; \widehat P^k}^{\pi^k}\big)(s_h^k,a_h^k)\log(1/\delta)} +  \log(1/\delta) \Big\}\label{eqn:event 2}\,,\\
    \Ecal_3&:=\Ecal_1\cap\{\forall m\in[0,\lceil\log_2(\frac{KH}{G})\rceil]: C_m\lesssim 2^{m}G+\sqrt{\log(1/\delta)\cdot C_{m+1}}+\log(1/\delta)\}\,,\label{eqn:event 3}\\
    \Ecal_4&:=    \{\sum_{k=0}^{K-1}\sum_{h=0}^{H-1}\left[\big(\VV_{P^\star} V_{h+1}^{\pi^k}\big)(s_h^k,a_h^k)\right]\lesssim\sum_{k=0}^{K-1}\var_{\pi^k}+\log(1/\delta)\}\,,\label{eqn:event 4}\\
    \Ecal_5&:=\{\sum_{k=0}^{K-1}\sum_{h=1}^H r(s_h^k, a_h^k)-\sum_{k=0}^{K-1} V^{\pi^k}_{0;P^*}\lesssim \sqrt{\sum_{k=0}^{K-1}\var_{\pi^k}\log(1/\delta)}+\log(1/\delta)\}\,,\label{eqn:event 5}\\
    \Ecal&:=\Ecal_2\cap\Ecal_3\cap\Ecal_4\cap\Ecal_5\,.
\end{align}

   First, by the realizability assumption, the standard generalization bound for MLE (\pref{lem:mle_generalization_offline}) with simply setting $D_i$ to be the delta distribution on the realized $(s_h^k, a_h^k)$ pairs, and a union bound over $K$ episodes, we have that w.p. at least $1-\delta$, for any $k\in[0,K-1]$: \\
    \begin{enumerate}
    \item[(1)] $P^\star\in\widehat\Pcal^k$; 
    \item[(2)]         \begin{equation}
            \sum_{i=0}^{k-1}\sum_{h=0}^{H-1}\mathbb H^2(P^\star(s_h^i,a_h^i)||\widehat P^k(s_h^i,a_h^i))\leq22\log(K\left|\Pcal\right|/\delta).\label{eqn: generalization online}
        \end{equation}
        \end{enumerate}
    This directly indicates that 
    \begin{equation}
        P(\mathbb{I}\{\Ecal_1\})\geq 1-\delta\,.\label{eqn: event 1 prob}
    \end{equation}
    Under event $\Ecal_1$, with the realizability in above (1), and by the optimistic algorithm design \\$(\pi^k,\widehat{P}^k)\leftarrow \argmax_{\pi\in\Pi, P\in\widehat\Pcal^k} V_{0; P}^\pi(s_0)$, for any $k\in[0,K-1]$, we have the following optimism guarantee
\begin{align}
    V^\star_{0;P^\star}&\leq \max_{\pi\in\Pi,P\in\widehat \Pcal^k} V^\pi_{0;P}=V^{\pi^k}_{0;\widehat P^k}.\notag
\end{align}

Then, under event $\Ecal_1$, we use \pref{lem: new eluder pigeon lemma} and \pref{eqn: generalization online} to get the following:

There exists a set $\Kcal\subseteq[K-1]$ such that 
\begin{itemize}
    \item $|\Kcal| \leq 13\log^2(88\log(K\left|\Pcal\right|/\delta) KH)\cdot DE_1(\Psi, \Scal \times \Acal, 1/(176\log(K\left|\Pcal\right|/\delta)KH))$
    \item And
    \begin{align}
        &\sum_{k \in [K-1]\setminus \Kcal}\sum_{h=0}^{H-1} \mathbb H^2\Big(P^\star(s_h^k,a_h^k)\Mid \widehat P^k\big(s_h^k,a_h^k)\Big)\notag \\
        &\leq \text{DE}_1(\Psi,\Scal \times \Acal,1/KH)\cdot(2 + 154\log(K\left|\Pcal\right|/\delta)\log(KH)) + 1\notag\\
        &\lesssim \text{DE}_1(\Psi,\Scal \times \Acal,1/KH)\cdot\log(K\left|\Pcal\right|/\delta)\log(KH)\,.\label{eqn: mle}
    \end{align}
\end{itemize}
We upper bound the regret with optimism, and by dividing $k\in[K-1]$ into $\Kcal$ and $[K-1]\setminus \Kcal$ with the assumption that the trajectory-wise cumulative reward is normalized in [0,1], as follows

\begin{align}
    &\sum_{k=0}^{K-1} V^\star_{0;P^\star}-\sum_{k=0}^{K-1}\sum_{h=1}^H r(s_h^k, a_h^k) \notag\\
    &\leq |\Kcal| +\sum_{k\in[K-1]\setminus \Kcal}\left(V^{\pi^k}_{0;\widehat P^k}-\sum_{h=0}^{H-1}r(s_h^k, a_h^k)\right)\notag\\
    &\lesssim  \log^2( \log(K\left|\Pcal\right|/\delta) KH)\cdot DE_1(\Psi, \Scal \times \Acal, 1/( \log(K\left|\Pcal\right|/\delta)KH))+\sum_{k\in[K-1]\setminus \Kcal}\left(V^{\pi^k}_{0;\widehat P^k}-\sum_{h=0}^{H-1}r(s_h^k, a_h^k)\right)\label{eqn: regret 1}\,.
\end{align}
We then do the following decomposition. Note that for any $k\in[K-1]$, policy $\pi^k$ is deterministic. We have that for any $k\in[K-1]$
\begin{align}
    & V^{\pi^k}_{0;\widehat P^k}(s_0^k)-\sum_{h=0}^{H-1} r(s_h^k, a_h^k) \notag \\
    &= Q^{\pi^k}_{0;\widehat P^k}(s_0^k, a_0^k)-\sum_{h=0}^{H-1} r(s_h^k, a_h^k)\notag \\
        &= r(s_0^k,a_0^k)+\EE_{s'\sim \widehat P^k(s_0^k, a_0^k)}V^{\pi^k}_{1;\widehat P^k}(s')-\sum_{h=0}^{H-1} r(s_h^k, a_h^k)\notag \\
    & = \EE_{s'\sim \widehat P^k(s_0^k, a_0^k)}V^{\pi^k}_{1;\widehat P^k}(s') - \sum_{h=1}^H r(s_h^k, a_h^k)\notag \\
    & = \EE_{s'\sim P^*(s_0^k, a_0^k)}V^{\pi^k}_{1;\widehat P^k}(s') - \sum_{h=1}^H r(s_h^k, a_h^k) + \EE_{s'\sim \widehat P^k(s_0^k, a_0^k)}V^{\pi^k}_{1;\widehat P^k}(s') - \EE_{s'\sim P^*(s_0^k, a_0^k)}V^{\pi^k}_{1;\widehat P^k}(s')\notag \\
    & = V^{\pi^k}_{1;\widehat P^k}(s_1^k)-\sum_{h=1}^{H-1} r(s_h^k, a_h^k) + \underbrace{\EE_{s'\sim P^*(s_0^k, a_0^k)}V^{\pi^k}_{1;\widehat P^k}(s')  - V^{\pi^k}_{1;\widehat P^k}(s_1^k)}_{I_0^k}\notag \\
    &\quad + \EE_{s'\sim \widehat P^k(s_0^k, a_0^k)}V^{\pi^k}_{1;\widehat P^k}(s') - \EE_{s'\sim P^*(s_0^k, a_0^k)}V^{\pi^k}_{1;\widehat P^k}(s')\notag\,,
\end{align}
where we use the Bellman equation for several times.

Then, by doing this recursively, we can get for any $k\in[K-1]$
\begin{align}
    &V^{\pi^k}_{0;\widehat P^k}(s_h^k)-\sum_{h=0}^{H-1} r(s_h^k, a_h^k) \notag\\
    &\leq \sum_{h=0}^{H-1} I_h^k + \sum_{h=0}^{H-1}  \left|\EE_{s'\sim \widehat P^k(s_h^k, a_h^k)}V^{\pi^k}_{h+1;\widehat P^k}(s') - \EE_{s'\sim P^*(s_h^k, a_h^k)}V^{\pi^k}_{{h+1};\widehat P^k}(s')\right|\label{eqn:regret2}
\end{align}
Therefore, 
\begin{align}
    &\sum_{k\in[K-1]\setminus \Kcal}(V^{\pi^k}_{0;\widehat P^k}(s_h^k)-\sum_{h=0}^{H-1} r(s_h^k, a_h^k)) \notag\\
    &\leq\sum_{k\in[K-1]\setminus \Kcal}\sum_{h=0}^{H-1} I_h^k + \sum_{k\in[K-1]\setminus \Kcal}\sum_{h=0}^{H-1}  \left|\EE_{s'\sim \widehat P^k(s_h^k, a_h^k)}V^{\pi^k}_{h+1;\widehat P^k}(s') - \EE_{s'\sim P^*(s_h^k, a_h^k)}V^{\pi^k}_{{h+1};\widehat P^k}(s')\right|\label{eqn:regret2}
\end{align}
Next we bound $\sum_{k\in[K-1]\setminus \Kcal}\sum_{h=0}^{H-1} I_h^k $. Note that by Azuma Bernstein's inequality, with probability at least $1-\delta$
\begin{align}
    \sum_{k\in[K-1]\setminus \Kcal}\sum_{h=0}^{H-1} I_h^k \leq \sqrt{2\sum_{k\in[K-1]\setminus \Kcal}\sum_{h=0}^{H-1} \big(\VV_{P^\star} V_{h+1; \widehat P^k}^{\pi^k}\big)(s_h^k,a_h^k)\log(1/\delta)} +  \frac{2}{3}\log(1/\delta) \label{eqn:regret 3}
\end{align}
This directly indicates that 
\begin{equation}
    P(\mathbb I\{\Ecal_2\})\geq 1-\delta\,. \label{eqn:event 2 probab}
\end{equation}
Then, we propose the following lemma.
\begin{lemma}[Bound of sum of mean value differences for online RL]\label{lem: online sum mean value difference}
    Under event $\Ecal_1$, we have 
    \begin{align}
        &\sum_{k\in[K-1]\setminus \Kcal}\sum_{h=0}^{H-1}    \left|\EE_{s'\sim \widehat P^k(s_h^k, a_h^k)}V^{\pi^k}_{h+1;\widehat P^k}(s') - \EE_{s'\sim P^*(s_h^k, a_h^k)}V^{\pi^k}_{{h+1};\widehat P^k}(s')\right| \notag\\
        &\lesssim\sqrt{\sum_{k\in[K-1]\setminus \Kcal}\sum_{h=0}^{H-1}  \left[\big(\VV_{P^\star} V_{h+1; \widehat P^k}^{\pi^k}\big) (s_h^k, a_h^k)\right]\cdot\text{DE}_1(\Psi,\Scal \times \Acal,1/KH)\cdot\log(K\left|\Pcal\right|/\delta)\log(KH)}\notag\\
        &\quad+\text{DE}_1(\Psi,\Scal \times \Acal,1/KH)\cdot\log(K\left|\Pcal\right|/\delta)\log(KH).\notag
    \end{align}
\end{lemma}
\begin{proof}[Proof of \pref{lem: online sum mean value difference}]  
Under event $\Ecal_1$, we have
    \begin{align}
    &\sum_{k\in[K-1]\setminus \Kcal}\sum_{h=0}^{H-1} \left|\EE_{s'\sim \widehat P^k(s_h^k, a_h^k)}V^{\pi^k}_{h+1;\widehat P^k}(s') - \EE_{s'\sim P^*(s_h^k, a_h^k)}V^{\pi^k}_{{h+1};\widehat P^k}(s')\right|\notag\\
    &\leq 4\sum_{k\in[K-1]\setminus \Kcal}\sum_{h=0}^{H-1} \left[\sqrt{\big(\VV_{P^\star} V_{h+1; \widehat P^k}^{\pi^k}\big)(s_h^k, a_h^k)D_\triangle\Big(V_{h+1; \widehat P^k}^{\pi^k}\big(s'\sim P^\star(s_h^k, a_h^k)\big)\Mid  V_{h+1; \widehat P^k}^{\pi^k}(s'\sim \widehat P^k\big(s_h^k, a_h^k)\big)\Big)}\right]\notag\\
&\quad+5\sum_{k\in[K-1]\setminus \Kcal}\sum_{h=0}^{H-1}\left[D_\triangle\Big(V_{h+1; \widehat P^k}^{\pi^k}\big(s'\sim P^\star(s_h^k, a_h^k)\big)\Mid V_{h+1; \widehat P^k}^{\pi^k}(s'\sim \widehat P^k\big(s_h^k, a_h^k)\big)\Big)\right]\notag\\
&\leq 8\sum_{k\in[K-1]\setminus \Kcal}\sum_{h=0}^{H-1} \left[\sqrt{\big(\VV_{P^\star} V_{h+1; \widehat P^k}^{\pi^k}\big)(s_h^k, a_h^k)\mathbb H^2\Big(V_{h+1; \widehat P^k}^{\pi^k}\big(s'\sim P^\star(s_h^k, a_h^k)\big)\Mid  V_{h+1; \widehat P^k}^{\pi^k}(s'\sim \widehat P^k\big(s_h^k, a_h^k)\big)\Big)}\right]\notag\\
&\quad+20\sum_{k\in[K-1]\setminus \Kcal}\sum_{h=0}^{H-1} \left[\mathbb H^2\Big(V_{h+1; \widehat P^k}^{\pi^k}\big(s'\sim  P^\star(s_h^k, a_h^k)\big)\Mid V_{h+1; \widehat P^k}^{\pi^k}(s'\sim \widehat P^k\big(s_h^k, a_h^k)\big)\Big)\right]\notag\\
            &\leq 8\sum_{k\in[K-1]\setminus \Kcal}\sum_{h=0}^{H-1}   \left[\sqrt{\big(\VV_{P^\star} V_{h+1; \widehat P^k}^{\pi^k}\big)(s_h^k, a_h^k)\mathbb H^2\Big(P^\star(s_h^k, a_h^k)\Mid \widehat P^k\big(s_h^k, a_h^k)\Big)}\right]\notag \\
            &\quad +20\sum_{k\in[K-1]\setminus \Kcal}\sum_{h=0}^{H-1} \left[\mathbb H^2\Big(P^\star(s_h^k, a_h^k)\Mid \widehat P^k\big(s_h^k, a_h^k)\Big)\right]\notag\\
                        &\leq 8 \sqrt{\sum_{k\in[K-1]\setminus \Kcal}\sum_{h=0}^{H-1} \left[\big(\VV_{P^\star} V_{h+1; \widehat P^k}^{\pi^k}\big)(s_h^k, a_h^k)\right]\cdot\sum_{k\in[K-1]\setminus \Kcal}\sum_{h=0}^{H-1}  \left[\mathbb H^2\Big(P^\star(s_h^k, a_h^k)\Mid \widehat P^k\big(s_h^k, a_h^k)\Big)\right]}\notag\\
&\quad+20\sum_{k\in[K-1]\setminus \Kcal}\sum_{h=0}^{H-1} \left[\mathbb H^2\Big(P^\star (s_h^k, a_h^k)\Mid \widehat P^k\big (s_h^k, a_h^k)\Big)\right]\notag\\
&\lesssim   \sqrt{\sum_{k\in[K-1]\setminus \Kcal}\sum_{h=0}^{H-1}  \left[\big(\VV_{P^\star} V_{h+1; \widehat P^k}^{\pi^k}\big) (s_h^k, a_h^k)\right]\cdot \text{DE}_1(\Psi,\Scal \times \Acal,1/KH)\cdot\log(K\left|\Pcal\right|/\delta)\log(KH)}\notag\\
&\quad+\text{DE}_1(\Psi,\Scal \times \Acal,1/KH)\cdot\log(K\left|\Pcal\right|/\delta)\log(KH)\label{eqn: mean to variance bound}\,,
\end{align}
where in the first inequality, we use \pref{lem: mean to variance} to bound the difference of two means $\EE_{s'\sim P^\star (s_h^k, a_h^k)} V^{\pi^k}_{h+1;\widehat P^k}(s') - \EE_{s'\sim \widehat P^k (s_h^k, a_h^k)} V^{\pi^*}_{h+1;\widehat P^k}(s')$ using variances and the triangle discrimination; in the second inequality we use the fact that that triangle discrimination is equivalent to squared Hellinger distance, i.e., $D_\triangle \leq 4 \mathbb H^2$; the third inequality is via data processing inequality on the squared Hellinger distance; the fourth inequality is by the Cauchy–Schwarz inequality; the last inequality holds under $\Ecal_1$ by \pref{eqn: mle}.
\end{proof}

The next lemma shows that the event $\mathcal{E}_3$ happens with high probability.

\begin{lemma}[Recursion Event Lemma]\label{lem: event 3 lemma}
Event $\Ecal_3$ happens with high probability. Specifically, we have
\begin{equation}
    P(\mathbb I\{\Ecal_3\})\geq 1-(1+\lceil\log_2(\frac{KH}{G})\rceil)\delta.
\end{equation}
\end{lemma}
\begin{proof}[Proof of \pref{lem: event 3 lemma}]
    Let $\Delta_{h+1}^{\pi^k}:=V_{h+1; \widehat P^k}^{\pi^k}-V_{h+1}^{\pi^k}$. First, under event $\Ecal_1$, with happens with probability at least $1-\delta$ by \pref{eqn: event 1 prob}, and also note that $\pi^k$ is deterministic for any $k\in[K-1]$, we can prove the following
\begin{align}
& \sum_{k\in[K-1]\setminus \Kcal}\sum_{h=0}^{H-1} \left[\left|(\Delta_{h}^{\pi^k})(s_{h}^k)-\big(P^\star\Delta_{h+1}^{\pi^k}\big)(s_h^k,a_h^k)\right|\right] \\
    &= \sum_{k\in[K-1]\setminus \Kcal}\sum_{h=0}^{H-1} \left[\left|({V}_{h;\widehat P^k}^{\pi^k})(s_{h}^k)-\big(P^\star{V}_{h+1;\widehat P^k}^{\pi^k}\big)(s_h^k,a_h^k)-\Big(({V}_{h}^{\pi^k})(s_{h}^k)-\big(P^\star{V}_{h+1}^{\pi^k}\big)(s_h^k,a_h^k)\Big)\right|\right]\notag\\
    &=\sum_{k\in[K-1]\setminus \Kcal}\sum_{h=0}^{H-1} \left[\left|r(s_h^k,a_h^k)+\big(\widehat P^k{V}_{h+1;\widehat P^k}^{\pi^k}\big)(s_h^k,a_h^k)-\big(P^\star{V}_{h+1;\widehat P^k}^{\pi^k}\big)(s_h^k,a_h^k)-r(s_h^k,a_h^k)\right|\right]\notag\\
        &=\sum_{k\in[K-1]\setminus \Kcal}\sum_{h=0}^{H-1} \left[\left|\big(\widehat P^k{V}_{h+1;\widehat P^k}^{\pi^k}\big)(s_h^k,a_h^k)-\big(P^\star{V}_{h+1;\widehat P^k}^{\pi^k}\big)(s_h^k,a_h^k)\right|\right]\notag\\
        &= \sum_{k\in[K-1]\setminus \Kcal}\sum_{h=0}^{H-1} \left[\left|\EE_{s'\sim P^\star(s_h^k,a_h^k)}\left[V^{\pi^k}_{h+1;\widehat P^k}(s')\right]-\EE_{s'\sim\widehat P^k(s_h^k,a_h^k)}\left[V^{\pi^k}_{h+1;\widehat P^k}(s')\right]\right|\right]\notag\\
        &\lesssim   \sqrt{\sum_{k\in[K-1]\setminus \Kcal}\sum_{h=0}^{H-1}  \left[\big(\VV_{P^\star} V_{h+1; \widehat P^k}^{\pi^k}\big) (s_h^k, a_h^k)\right]\cdot \text{DE}_1(\Psi,\Scal \times \Acal,1/KH)\cdot\log(K\left|\Pcal\right|/\delta)\log(KH)}\notag\\
&\quad+\text{DE}_1(\Psi,\Scal \times \Acal,1/KH)\cdot\log(K\left|\Pcal\right|/\delta)\log(KH)\notag\\
        &=G\label{eqn: help 11}
\end{align}
 where the first equality is by the definition of $\Delta_{h+1}^{\pi^k}$, the inequality holds under $\Ecal_1$ by \pref{lem: online sum mean value difference}, and the last equality is by definition of $A$ and $G$.

Under event $\Ecal_1$, with probability at least $1-\lceil\log_2(\frac{KH}{G})\rceil\delta$, for any $m\in[0,\lceil\log_2(\frac{KH}{G})\rceil]$

\begin{align}
    C_m&=\sum_{k\in[K-1]\setminus \Kcal}\sum_{h=0}^{H-1} \left[\big(\VV_{P^\star}( V_{h+1; \widehat P^k}^{\pi^k}-V_{h+1}^{\pi^k})^{2^m}\big)(s_h^k,a_h^k)\right]\notag\\
    &=\sum_{k\in[K-1]\setminus \Kcal}\sum_{h=0}^{H-1} \left[\big(P^\star(\Delta_{h+1}^{\pi^k})^{2^{m+1}}\big)(s_h^k,a_h^k)-\big((P^\star(\Delta_{h+1}^{\pi^k})^{2^{m}})(s_h^k,a_h^k)\big)^2\right]\notag\\
    &=\sum_{k\in[K-1]\setminus \Kcal}\sum_{h=0}^{H-1}\left[(\Delta_{h+1}^{\pi^k})^{2^{m+1}}(s_{h+1}^k)\right]-\sum_{k\in[K-1]\setminus \Kcal}\sum_{h=0}^{H-1} \left[\big((P^\star(\Delta_{h+1}^{\pi^k})^{2^{m}})(s_h^k,a_h^k)\big)^2\right]\notag\\
    &\quad+\sum_{k\in[K-1]\setminus \Kcal}\sum_{h=0}^{H-1}\left(\EE_{s\sim P^*(s_h^k,a_h^k)}\left[(\Delta_{h+1}^{\pi^k})^{2^{m+1}}(s)\right]-(\Delta_{h+1}^{\pi^k})^{2^{m+1}}(s_{h+1}^k)\right)\notag\\
    &\leq \sum_{k\in[K-1]\setminus \Kcal}\sum_{h=0}^{H-1} \left[(\Delta_{h}^{\pi^k})^{2^{m+1}}(s_h^k)-\big((P^\star(\Delta_{h+1}^{\pi^k})^{2^{m}})(s_h^k,a_h^k)\big)^2\right]\notag\\
    &\quad+\sum_{k\in[K-1]\setminus \Kcal}\sum_{h=0}^{H-1}\left(\EE_{s\sim P^*(s_h^k,a_h^k)}\left[(\Delta_{h+1}^{\pi^k})^{2^{m+1}}(s)\right]-(\Delta_{h+1}^{\pi^k})^{2^{m+1}}(s_h^k)\right)\notag\\
    &\lesssim \sum_{k\in[K-1]\setminus \Kcal}\sum_{h=0}^{H-1} \left[(\Delta_{h}^{\pi^k})^{2^{m+1}}(s_h^k)-\big((P^\star(\Delta_{h+1}^{\pi^k})^{2^{m}})(s_h^k,a_h^k)\big)^2\right]+\log(1/\delta)\notag\\
    &\quad+\sqrt{\sum_{k\in[K-1]\setminus \Kcal}\sum_{h=0}^{H-1} \VV_{P^*}\left((V_{h+1; \widehat P^k}^{\pi^k}-V_{h+1}^{\pi^k})^{2^{m+1}}\right)(s_h^k,a_h^k)\log(1/\delta)}\notag\\
            &=\sum_{k\in[K-1]\setminus \Kcal}\sum_{h=0}^{H-1} \left[\Big((\Delta_{h}^{\pi^k})^{2^m}(s_h^k)+(P^\star(\Delta_{h+1}^{\pi^k})^{2^{m}})(s_h^k,a_h^k)\Big)\cdot\Big((\Delta_{h}^{\pi^k})^{2^m}(s_h^k)-(P^\star(\Delta_{h+1}^{\pi^k})^{2^{m}})(s_h^k,a_h^k)\Big)\right] \notag\\
    &\quad+\sqrt{\log(1/\delta)\cdot C_{m+1}}+\log(1/\delta)\notag\\
    &=\sum_{k\in[K-1]\setminus \Kcal}\sum_{h=0}^{H-1} \left[\Big((\Delta_{h}^{\pi^k})^{2^m}(s_h^k)+(P^\star(\Delta_{h+1}^{\pi^k})^{2^{m}})(s_h^k,a_h^k)\Big)\cdot\Big((\Delta_{h}^{\pi^k})^{2^m}(s_h^k)-(P^\star((\Delta_{h+1}^{\pi^k})^2)^{2^{m-1}})(s_h^k,a_h^k)\Big)\right] \notag\\
    &\quad+\sqrt{\log(1/\delta)\cdot C_{m+1}}+\log(1/\delta)\notag\\
    &\leq\sum_{k\in[K-1]\setminus \Kcal}\sum_{h=0}^{H-1} \left[\Big((\Delta_{h}^{\pi^k})^{2^m}(s_h^k)+(P^\star(\Delta_{h+1}^{\pi^k})^{2^{m}})(s_h^k,a_h^k)\Big)\cdot\Big((\Delta_{h}^{\pi^k})^{2^m}(s_h^k)-((P^\star (\Delta_{h+1}^{\pi^k})^2)(s_h^k,a_h^k))^{2^{m-1}}\Big)\right] \notag\\
    &\quad+\sqrt{\log(1/\delta)\cdot C_{m+1}}+\log(1/\delta)\notag\\
 &\leq 2^{m} \sum_{k\in[K-1]\setminus \Kcal}\sum_{h=0}^{H-1} \left[\left|(\Delta_{h}^{\pi^k})^2(s_h^k)-((P^\star \Delta_{h+1}^{\pi^k})(s_h^k,a_h^k))^2\right|\right]+ \sqrt{\log(1/\delta)\cdot C_{m+1}}+\log(1/\delta)\notag\\
 &=2^{m} \sum_{k\in[K-1]\setminus \Kcal}\sum_{h=0}^{H-1} \left[\left|\big((\Delta_{h}^{\pi^k})(s_h^k)+(P^\star \Delta_{h+1}^{\pi^k})(s_h^k,a_h^k)\big)\cdot\big((\Delta_{h}^{\pi^k})(s_h^k)-(P^\star \Delta_{h+1}^{\pi^k})(s_h^k,a_h^k)\big)\right|\right]\notag\\
 &\quad+ \sqrt{\log(1/\delta)\cdot C_{m+1}}+\log(1/\delta)\notag\\
    &\leq 2\cdot2^{m}\sum_{k\in[K-1]\setminus \Kcal}\sum_{h=0}^{H-1} \left[\left|(\Delta_{h}^{\pi^k}) (s_h^k)-\big(P^\star\Delta_{h+1}^{\pi^k}\big) (s_h^k,a_h^k)\right|\right]+\sqrt{\log(1/\delta)\cdot C_{m+1}}+\log(1/\delta)\notag\\
    &\lesssim 2^{m}G+\sqrt{\log(1/\delta)\cdot C_{m+1}}+\log(1/\delta)
    \label{eqn: help 10}\,,
\end{align}
where in the first inequality we change the index, the second inequality holds with probability at least $1-\delta$ by Azuma Bernstain's inequality, the third inequality holds because that $E[X^{2^{m-1}}]\geq (E[X])^{2^{m-1}}$ for $m\geq 1$ and $X\geq 0$, the fourth inequality holds by keep using $a^2-b^2=(a+b)(a-b)$, then with $E[X^2]\geq E[X]^2$, and the assumption that the trajectory-wise total reward is normalized in $[0,1]$, the last inequality holds under $\Ecal_1$ by \pref{eqn: help 11}, and we take a union bound to get this hold for all $m\in[0,\lceil\log_2(\frac{KH}{G})\rceil]$ with probability at least $1-\lceil\log_2(\frac{KH}{G})\rceil\delta$ (because for each $m\in[0,\lceil\log_2(\frac{KH}{G})\rceil]$ we need to apply the Azuma Bernstain's inequality once).

The above reasoning directly implies that
\begin{equation}
    P(\mathbb I\{\Ecal_3\})\geq 1-(1+\lceil\log_2(\frac{KH}{G})\rceil)\delta.
\end{equation}
\end{proof}

Under the event $\mathcal{E}_3$, we prove the following lemma to bound $\sum_{k\in[K-1]\setminus \Kcal}\sum_{h=0}^{H-1}  \left[\big(\VV_{P^\star} V_{h+1; \widehat P^k}^{\pi^k}\big) (s_h^k, a_h^k)\right]$.

\begin{lemma}[Variance Conversion Lemma for online RL]\label{lem: variance recusion online}
Under event $\Ecal_3$, we have
\begin{align}
    &\sum_{k\in[K-1]\setminus \Kcal}\sum_{h=0}^{H-1}  \left[\big(\VV_{P^\star} V_{h+1; \widehat P^k}^{\pi^k}\big) (s_h^k, a_h^k)\right]\notag\\
    &\leq O\Big(\sum_{k\in[K-1]\setminus \Kcal}\sum_{h=0}^{H-1} \left[\big(\VV_{P^\star} V_{h+1}^{\pi^k}\big)(s_h^k, a_h^k)\right]+\text{DE}_1(\Psi,\Scal \times \Acal,1/KH)\cdot\log(K\left|\Pcal\right|/\delta)\log(KH)\Big).\notag
\end{align}
\end{lemma}
\begin{proof}[Proof of \pref{lem: variance recusion online}] 
Under $\Ecal_3$, we have for any $m\in[0,\lceil\log_2(\frac{KH}{G})\rceil]$
\begin{equation}
    C_m\lesssim 2^{m}G+\sqrt{\log(1/\delta)\cdot C_{m+1}}+\log(1/\delta)\,.
\end{equation}
Then, by \pref{lem:recursion bound C_m}, we have 
\begin{equation}
    C_0\lesssim G\,.
\end{equation}
Also note that we have $A\leq 2B+2C_0$ since $\VV_{P^\star}(a+b)\leq 2\VV_{P^\star}(a)+2\VV_{P^\star}(b)$.
Therefore, we have
\begin{align}
    A&\leq 2B+2C_0\notag\\
    &\lesssim  B+ G\notag\\
 &=  B+\sqrt{A\cdot \text{DE}_1(\Psi,\Scal \times \Acal,1/KH)\cdot\log(K\left|\Pcal\right|/\delta)\log(KH)}\notag\\
&\quad+\text{DE}_1(\Psi,\Scal \times \Acal,1/KH)\cdot\log(K\left|\Pcal\right|/\delta)\log(KH)
\end{align}

Then, with the fact that $x\leq 2a+b^2$ if $x\leq a+b\sqrt{x}$, we have
\begin{align}
    A &\leq O\Bigg(B+\text{DE}_1(\Psi,\Scal \times \Acal,1/KH)\cdot\log(K\left|\Pcal\right|/\delta)\log(KH)\Bigg)\,,
\end{align}
which is
\begin{align}
    &\sum_{k\in[K-1]\setminus \Kcal}\sum_{h=0}^{H-1} \left[\big(\VV_{P^\star} V_{h+1; \widehat P^k}^{\pi^k}\big)(s_h^k,a_h^k)\right]\notag\\
    &\leq O\bigg(\sum_{k\in[K-1]\setminus \Kcal}\sum_{h=0}^{H-1}\left[\big(\VV_{P^\star} V_{h+1}^{\pi^k}\big)(s_h^k,a_h^k)\right]+\text{DE}_1(\Psi,\Scal \times \Acal,1/KH)\cdot\log(K\left|\Pcal\right|/\delta)\log(KH)\bigg)
\end{align}
\end{proof}



By the same reasoning in Lemma 26 of \cite{zhou2023sharp}, we have that with probability at least $1-\delta$
\begin{align}
    \sum_{k\in[K-1]\setminus \Kcal}\sum_{h=0}^{H-1}\left[\big(\VV_{P^\star} V_{h+1}^{\pi^k}\big)(s_h^k,a_h^k)\right]&\leq \sum_{k=0}^{K-1}\sum_{h=0}^{H-1}\left[\big(\VV_{P^\star} V_{h+1}^{\pi^k}\big)(s_h^k,a_h^k)\right]\notag \\
    &\leq O(\sum_{k=0}^{K-1}\var_{\pi^k}+\log(1/\delta))\,.
\end{align}
This indicates that 
\begin{equation}
    P(\mathbb I\{\Ecal_4\})\geq 1-\delta\,.
\end{equation}
We can use the Azuma Bernstain's inequality to get that with probability at least $1-\delta$:
\begin{align}
    \sum_{k=0}^{K-1}\sum_{h=1}^H r(s_h^k, a_h^k)-\sum_{k=0}^{K-1} V^{\pi^k}_{0;P^*}\lesssim \sqrt{\sum_{k=0}^{K-1}\var_{\pi^k}\log(1/\delta)}+\log(1/\delta)\,.
\end{align}
This indicates that 
\begin{equation}
    P(\mathbb I\{\Ecal_5\})\geq 1-\delta\,.
\end{equation}
Then, together with \pref{lem: event 3 lemma},  \pref{eqn: event 1 prob} and \pref{eqn:event 2 probab}, we have
\begin{equation}
    P(\mathbb I\{\Ecal\})\geq 1-(5+\lceil\log_2(\frac{KH}{G})\rceil)\delta\geq 1-5KH\delta\,. \label{eqn: event prob bound}
\end{equation}

Finally, under event $\Ecal$, with all the things above (\pref{eqn: regret 1}, \pref{eqn:regret2}, \pref{eqn:regret 3},\pref{lem: online sum mean value difference}, \pref{lem: variance recusion online}), we have
\begin{align}
    &\sum_{k=0}^{K-1} V^\star_{0;P^\star}-\sum_{k=0}^{K-1} V^{\pi^k}_{0;P^*}\notag\\
    &=\sum_{k=0}^{K-1} V^\star_{0;P^\star}-\sum_{k=0}^{K-1}\sum_{h=1}^H r(s_h^k, a_h^k) +\sum_{k=0}^{K-1}\sum_{h=1}^H r(s_h^k, a_h^k)-\sum_{k=0}^{K-1} V^{\pi^k}_{0;P^*}\notag\\
    &\lesssim |\Kcal| +\sum_{k\in[K-1]\setminus \Kcal}\left(V^{\pi^k}_{0;\widehat P^k}-\sum_{h=0}^{H-1}r(s_h^k, a_h^k)\right)+\sqrt{\sum_{k=0}^{K-1}\var_{\pi^k}\log(1/\delta)}+\log(1/\delta)\notag\\
    &\lesssim \log^2( \log(K\left|\Pcal\right|/\delta) KH)\cdot DE_1(\Psi, \Scal \times \Acal, 1/( \log(K\left|\Pcal\right|/\delta)KH))+\sum_{k\in[K-1]\setminus \Kcal}\left(V^{\pi^k}_{0;\widehat P^k}-\sum_{h=0}^{H-1}r(s_h^k, a_h^k)\right)\notag\\
    &\quad+\sqrt{\sum_{k=0}^{K-1}\var_{\pi^k}\log(1/\delta)}+\log(1/\delta)\notag\\
    &\lesssim \log^2( \log(K\left|\Pcal\right|/\delta) KH)\cdot DE_1(\Psi, \Scal \times \Acal, 1/( \log(K\left|\Pcal\right|/\delta)KH)) +\log(1/\delta)\notag\\
    &\quad+\sqrt{\sum_{k\in[K-1]\setminus \Kcal}\sum_{h=0}^{H-1} \big(\VV_{P^\star} V_{h+1; \widehat P^k}^{\pi^k}\big)(s_h^k,a_h^k)\log(1/\delta)}+\sqrt{\sum_{k=0}^{K-1}\var_{\pi^k}\log(1/\delta)}\notag\\
    &\quad+ \sum_{k\in[K-1]\setminus \Kcal}\sum_{h=0}^{H-1}    \left|\EE_{s'\sim \widehat P^k(s_1^k, A^k)}V^{\pi^k}_{1;\widehat P^k}(s') - \EE_{s'\sim P^*(s_1^k, A^k)}V^{\pi^k}_{1;\widehat P^k}(s')\right| \notag\\
        &\lesssim \log^2( \log(K\left|\Pcal\right|/\delta) KH)\cdot DE_1(\Psi, \Scal \times \Acal, 1/( \log(K\left|\Pcal\right|/\delta)KH))+\sqrt{\sum_{k=0}^{K-1}\var_{\pi^k}\log(1/\delta)}+\log(1/\delta)\notag\\
        &+\sqrt{(\sum_{k\in[K-1]\setminus \Kcal}\sum_{h=0}^{H-1} \left[\big(\VV_{P^\star} V_{h+1}^{\pi^k}\big)(s_h^k, a_h^k)\right]+\text{DE}_1(\Psi,\Scal \times \Acal,1/KH)\cdot\log(K\left|\Pcal\right|/\delta)\log(KH))\cdot\log(1/\delta) }\notag\\
        &\quad+\sqrt{\sum_{k\in[K-1]\setminus \Kcal}\sum_{h=0}^{H-1}  \left[\big(\VV_{P^\star} V_{h+1; \widehat P^k}^{\pi^k}\big) (s_h^k, a_h^k)\right]\cdot\text{DE}_1(\Psi,\Scal \times \Acal,1/KH)\cdot\log(K\left|\Pcal\right|/\delta)\log(KH)}\notag\\
        &\lesssim \log^2( \log(K\left|\Pcal\right|/\delta) KH)\cdot DE_1(\Psi, \Scal \times \Acal, 1/( \log(K\left|\Pcal\right|/\delta)KH))+\sqrt{\sum_{k=0}^{K-1}\var_{\pi^k}\log(1/\delta)}+\log(1/\delta)\notag\\
        &+\sqrt{(\sum_{k\in[K-1]\setminus \Kcal}\sum_{h=0}^{H-1} \left[\big(\VV_{P^\star} V_{h+1}^{\pi^k}\big)(s_h^k, a_h^k)\right]+\text{DE}_1(\Psi,\Scal \times \Acal,1/KH)\cdot\log(K\left|\Pcal\right|/\delta)\log(KH))\cdot\log(1/\delta) }\notag\\
        &+\sqrt{(\sum_{k\in[K-1]\setminus \Kcal}\sum_{h=0}^{H-1} \left[\big(\VV_{P^\star} V_{h+1}^{\pi^k}\big)(s_h^k, a_h^k)\right]+\text{DE}_1(\Psi,\Scal \times \Acal,1/KH)\cdot\log(K\left|\Pcal\right|/\delta)\log(KH))}\notag\\
        &\quad\times\sqrt{\text{DE}_1(\Psi,\Scal \times \Acal,1/KH)\cdot\log(K\left|\Pcal\right|/\delta)\log(KH))}\notag\displaybreak[1]\\
        &\lesssim \log^2( \log(K\left|\Pcal\right|/\delta) KH)\cdot DE_1(\Psi, \Scal \times \Acal, 1/( \log(K\left|\Pcal\right|/\delta)KH))+\sqrt{\sum_{k=0}^{K-1}\var_{\pi^k}\log(1/\delta)}+\log(1/\delta)\notag\\
        &+\sqrt{(\sum_{k=0}^{K-1}\var_{\pi^k}+\log(1/\delta)+\text{DE}_1(\Psi,\Scal \times \Acal,1/KH)\cdot\log(K\left|\Pcal\right|/\delta)\log(KH))\cdot\log(1/\delta) }\notag\\
        &+\sqrt{\left( \sum_{k=0}^{K-1}\var_{\pi^k} + \log\left(\frac{1}{\delta}\right) + \text{DE}_1\left(\Psi, \Scal \times \Acal, \frac{1}{KH}\right) \cdot \log\left(\frac{K \left|\Pcal\right|}{\delta}\right) \log(KH) \right)}
\notag \\
&\quad\times \sqrt{\text{DE}_1\left(\Psi, \Scal \times \Acal, \frac{1}{KH}\right) \log\left(\frac{K \left|\Pcal\right|}{\delta}\right) \log(KH)}\notag\\
        &\leq O\Big(\sqrt{\sum_{k=0}^{K-1}\var_{\pi^k}\cdot\text{DE}_1(\Psi,\Scal \times \Acal,1/KH)\cdot\log(K\left|\Pcal\right|/\delta)\log(KH)}\notag\\
        &\quad+\text{DE}_1(\Psi,\Scal \times \Acal,1/KH)\cdot\log(K\left|\Pcal\right|/\delta)\log(KH)\Big)\,.\label{eqn: regret final 1}
\end{align}

The final result follows by replacing $\delta$ to be $\delta/(5KH)$ to make the event $\Ecal$ happen with probability at least $1-\delta$.
\subsection{Proof of  \pref{corr:online_coro_faster}}\label{app:online_coro_faster}
\begin{proof}[Proof of  \pref{corr:online_coro_faster}]
By \pref{lem:variance_lemma}, we have
\begin{equation}
    \var_{\pi^k}=\sum_{h=0}^{H-1}\EE_{s,a \sim d^{\pi^k}_h}\left[\big(\VV_{P^\star} V_{h+1}^{\pi^k}\big)(s,a)\right]
\end{equation}
Therefore, when $P^\star$ is deterministic, 
the $\EE_{s,a \sim d^{\pi^k}_h}\left[\big(\VV_{P^\star} V_{h+1}^{\pi^k}\big)(s,a)\right]$ terms are all 0 for any $k\in[K-1]$ and $h\in[H-1]$, and then the $\sum_{k=0}^{K-1}\var_{\pi^k}$ term in the higher order term in \pref{thm:online_theorem} is 0.
\end{proof}

\subsection{Proof of  \pref{corr:online_coro_infinite}}\label{app:online_coro_infinite}
\begin{proof}[Proof of  \pref{corr:online_coro_infinite}]
    We follow the MLE guarantee for the infinite model class in \pref{lem:mle_generalization infinite} and the same proof steps in the proof of \pref{thm:online_theorem} in Appendix \ref{app:online}. 
\end{proof}

\section{Detailed Proofs for the Offline RL setting in \pref{sec:offline}}\label{app: offline full}



\subsection{Proof of \pref{thm:mleoffline}}\label{app:offline}
The following is the full proof of \pref{thm:mleoffline}. 

\begin{proof}[Proof of \pref{thm:mleoffline}]
    First, by the realizability assumption, the standard generalization bound for MLE (\pref{lem:mle_generalization_offline}) with simply setting $D_i$ to be the delta distribution on the  $(s_h^k, a_h^k)$ pairs in the offline dataset $\mathcal{D}$, we have that w.p. at least $1-\delta$ : \\
    \begin{enumerate}
    \item[(1)] $P^\star\in\widehat\Pcal$; 
    \item[(2)]         \begin{equation}
            \frac{1}{K}\sum_{k=1}^K\sum_{h=0}^{H-1}\mathbb H^2(P^\star(s_h^k,a_h^k)||\widehat P(s_h^k,a_h^k))\leq\frac{22\log(\left|\Pcal\right|/\delta)}{K}.\label{eqn: generalization offline}
        \end{equation}
        \end{enumerate}
        
        
Then, with the above realizability in (1), and by the pessimistic algorithm design $\hat\pi\leftarrow \argmax_{\pi\in\Pi}\min_{P\in\widehat\Pcal} V_{0; P}^\pi(s_0)$, $\widehat P\leftarrow \argmin_{P \in \widehat \Pcal} V_{0; P}^{\widehat \pi}(s_0)$, we have that for any $\pi^\star\in\Pi$
\begin{align}
    V_{0; P^\star}^{\pi^\star}-V_{0; P^\star}^{\hat{\pi}}&=V_{0; P^\star}^{\pi^\star}-V_{0;\widehat P}^{\pi^\star}+V_{0;\widehat P}^{\pi^\star}-V_{0; P^\star}^{\hat{\pi}}\notag\\
    &\leq V_{0; P^\star}^{\pi^\star}-V_{0;\widehat P}^{\pi^\star}+V_{0;\widehat P}^{\hat{\pi}}-V_{0; P^\star}^{\hat{\pi}}\notag\\
    &\leq V_{0; P^\star}^{\pi^\star}-V_{0;\widehat P}^{\pi^\star}\,. \label{offline gap 1}
\end{align}
We can then bound $V_{0; P^\star}^{\pi^\star}-V_{0;\widehat P}^{\pi^\star}$ using the simulation lemma (\pref{lem:simulation}):
\begin{align}
    V_{0; P^\star}^{\pi^\star}-V_{0;\widehat P}^{\pi^\star}&\leq \sum_{h=0}^{H-1} \EE_{s,a\sim d^{\pi^\star}_h} \left[ \left|\EE_{s'\sim P^\star(s,a)} V^{\pi^\star}_{h+1;\widehat P}(s') -\EE_{s'\sim \widehat P(s,a)} V^{\pi^\star}_{h+1;\widehat P}(s')    \right|\right]\,. \label{eqn: offline gap 2}
\end{align}
Then, we prove the following lemma to bound the RHS of \pref{eqn: offline gap 2}.
\begin{lemma}[Bound of sum of mean value differences for offline RL]\label{lem: offline sum mean value difference}
    With probability at least $1-\delta$, we have 
    \begin{align}
        &\sum_{h=0}^{H-1} \EE_{s,a \sim d^{\pi^\star}_h} \left[ \left|\EE_{s'\sim P^\star(s,a)} V^{\pi^\star}_{h+1;\widehat P}(s') -\EE_{s'\sim \widehat P(s,a)} V^{\pi^\star}_{h+1;\widehat P}(s')    \right|\right]\notag\\
        &\leq 8\sqrt{\sum_{h=0}^{H-1} \EE_{s,a\sim d^{\pi^\star}_h}\left[\big(\VV_{P^\star} V_{h+1; \widehat P}^{\pi^\star}\big)(s,a)\right]\cdot\frac{22C^{\pi^\star}\log(\left|\Pcal\right|/\delta)}{K}}+\frac{440C^{\pi^\star}\log(\left|\Pcal\right|/\delta)}{K}.\notag
    \end{align}
\end{lemma}
\begin{proof}[Proof of \pref{lem: offline sum mean value difference}]  
We have
    \begin{align}
    &\sum_{h=0}^{H-1} \EE_{s,a \sim d^{\pi^\star}_h} \left[ \left|\EE_{s'\sim P^\star(s,a)} V^{\pi^\star}_{h+1;\widehat P}(s') -\EE_{s'\sim \widehat P(s,a)} V^{\pi^\star}_{h+1;\widehat P}(s')    \right|\right]\notag\\
    &\leq 4\sum_{h=0}^{H-1} \EE_{s,a \sim d^{\pi^\star}_h} \left[\sqrt{\big(\VV_{P^\star} V_{h+1; \widehat P}^{\pi^\star}\big)(s,a)D_\triangle\Big(V_{h+1; \widehat P}^{\pi^\star}\big(s'\sim P^\star(s,a)\big)\Mid  V_{h+1; \widehat P}^{\pi^\star}(s'\sim \widehat P\big(s,a)\big)\Big)}\right]\notag\\
&\quad+5\sum_{h=0}^{H-1}\EE_{s,a \sim d^{\pi^\star}_h}\left[D_\triangle\Big(V_{h+1; \widehat P}^{\pi^\star}\big(s'\sim P^\star(s,a)\big)\Mid V_{h+1; \widehat P}^{\pi^\star}(s'\sim \widehat P\big(s,a)\big)\Big)\right]\notag\\&\leq 8\sum_{h=0}^{H-1} \EE_{s,a \sim d^{\pi^\star}_h} \left[\sqrt{\big(\VV_{P^\star} V_{h+1; \widehat P}^{\pi^\star}\big)(s,a)\mathbb H^2\Big(V_{h+1; \widehat P}^{\pi^\star}\big(s'\sim P^\star(s,a)\big)\Mid  V_{h+1; \widehat P}^{\pi^\star}(s'\sim \widehat P\big(s,a)\big)\Big)}\right]\notag\\
&\quad+20\sum_{h=0}^{H-1}\EE_{s,a \sim d^{\pi^\star}_h}\left[\mathbb H^2\Big(V_{h+1; \widehat P}^{\pi^\star}\big(s'\sim  P^\star(s,a)\big)\Mid V_{h+1; \widehat P}^{\pi^\star}(s'\sim \widehat P\big(s,a)\big)\Big)\right]\notag\\
            &\leq 8\sum_{h=0}^{H-1} \EE_{s,a \sim d^{\pi^\star}_h} \left[\sqrt{\big(\VV_{P^\star} V_{h+1; \widehat P}^{\pi^\star}\big)(s,a)\mathbb H^2\Big(P^\star(s,a)\Mid \widehat P\big(s,a)\Big)}\right]+20\sum_{h=0}^{H-1}\EE_{s,a \sim d^{\pi^\star}_h}\left[\mathbb H^2\Big(P^\star(s,a)\Mid \widehat P\big(s,a)\Big)\right]\label{eqn:help 1}
\end{align}
where in the first inequality, we use \pref{lem: mean to variance} to bound the difference of two means $\EE_{s'\sim P^\star(s,a)} V^{\pi^\star}_{h+1;\widehat P}(s') - \EE_{s'\sim \widehat P(s,a)} V^{\pi^*}_{h+1;\widehat P}(s')$ using variances and the triangle discrimination; in the second inequality we use the fact that that triangle discrimination is equivalent to squared Hellinger distance, i.e., $D_\triangle \leq 4 \mathbb H^2$; the third inequality is via data processing inequality on the squared Hellinger distance. Next, starting from \pref{eqn:help 1}, with probability at least $1-\delta$, we have

\begin{align}
&8\sum_{h=0}^{H-1} \EE_{s,a \sim d^{\pi^\star}_h} \left[\sqrt{\big(\VV_{P^\star} V_{h+1; \widehat P}^{\pi^\star}\big)(s,a)\mathbb H^2\Big(P^\star(s,a)\Mid \widehat P\big(s,a)\Big)}\right]+20\sum_{h=0}^{H-1}\EE_{s,a \sim d^{\pi^\star}_h}\left[\mathbb H^2\Big(P^\star(s,a)\Mid \widehat P\big(s,a)\Big)\right]\notag \\
     &\leq 8 \sqrt{\sum_{h=0}^{H-1} \EE_{s,a \sim d^{\pi^\star}_h}\left[\big(\VV_{P^\star} V_{h+1; \widehat P}^{\pi^\star}\big)(s,a)\right]\cdot\sum_{h=0}^{H-1} \EE_{s,a \sim d^{\pi^\star}_h}\left[\mathbb H^2\Big(P^\star(s,a)\Mid \widehat P\big(s,a)\Big)\right]}\notag\\
&\quad+20\sum_{h=0}^{H-1}\EE_{s,a \sim d^{\pi^\star}_h}\left[\mathbb H^2\Big(P^\star(s,a)\Mid \widehat P\big(s,a)\Big)\right]\notag\\
&\leq 8 \sqrt{\sum_{h=0}^{H-1} \EE_{s,a \sim d^{\pi^\star}_h}\left[\big(\VV_{P^\star} V_{h+1; \widehat P}^{\pi^\star}\big)(s,a)\right]\cdot C^{\pi^\star}\frac{1}{K}\sum_{k=1}^K\sum_{h=0}^{H-1}\mathbb H^2(P^\star(s_h^k,a_h^k)||\widehat P(s_h^k,a_h^k))}\notag\\
&\quad+20C^{\pi^\star}\frac{1}{K}\sum_{k=1}^K\sum_{h=0}^{H-1}\mathbb H^2(P^\star(s_h^k,a_h^k)||\widehat P(s_h^k,a_h^k))\notag\\
&\leq 8 \sqrt{\sum_{h=0}^{H-1} \EE_{s,a \sim d^{\pi^\star}_h}\left[\big(\VV_{P^\star} V_{h+1; \widehat P}^{\pi^\star}\big)(s,a)\right]\cdot\frac{22C^{\pi^\star}\log(\left|\Pcal\right|/\delta)}{K}}+\frac{440C^{\pi^\star}\log(\left|\Pcal\right|/\delta)}{K}\label{eqn: proof mle offline 1}\,,
\end{align}
where the first inequality is by the Cauchy–Schwarz inequality; the second inequality is by the definition of single policy coverage (\pref{def: coverage offline}); the last inequality holds with probability at least $1-\delta$ with \pref{eqn: generalization offline}. Substituting \pref{eqn: proof mle offline 1} into \pref{eqn:help 1} ends our proof. 
\end{proof}



We denote $\mathcal{E}$ as the event that \pref{lem: offline sum mean value difference} holds. Under the event $\mathcal{E}$, we prove the following lemma to bound $\sum_{h=0}^{H-1} \EE_{s,a \sim d^{\pi^\star}_h}\left[\big(\VV_{P^\star} V_{h+1; \widehat P}^{\pi^\star}\big)(s,a)\right]$ with $\widetilde O(\sum_{h=0}^{H-1} \EE_{s,a\sim d^{\pi^*}_h}\left[\big(\VV_{P^\star} V_{h+1}^{\pi^*}\big)(s,a)\right]+C^{\pi^*}\log(\left|\Pcal\right|/\delta)/{K}$).

\begin{lemma}[Variance Conversion Lemma for offline RL]\label{lem: variance recusion}
Under event $\mathcal{E}$, we have
\begin{align}
    \sum_{h=0}^{H-1} \EE_{s,a \sim d^{\pi^\star}_h}\left[\big(\VV_{P^\star} V_{h+1; \widehat P}^{\pi^\star}\big)(s,a)\right]\leq O\Big(\sum_{h=0}^{H-1} \EE_{s,a\sim d^{\pi^*}_h}\left[\big(\VV_{P^\star} V_{h+1}^{\pi^*}\big)(s,a)\right]+C^{\pi^*}\frac{\log(\left|\Pcal\right|/\delta)}{K}\Big).\notag
\end{align}
\end{lemma}
\begin{proof}[Proof of \pref{lem: variance recusion}]
For notational simplicity, we denote $A:=\sum_{h=0}^{H-1} \EE_{s,a \sim d^{\pi^\star}_h}\left[\big(\VV_{P^\star} V_{h+1; \widehat P}^{\pi^\star}\big)(s,a)\right]$, and we denote \\
$B:=\sum_{h=0}^{H-1} \EE_{s,a \sim d^{\pi^\star}_h}\left[\big(\VV_{P^\star} V_{h+1}^{\pi^\star}\big)(s,a)\right]$, $C:=\sum_{h=0}^{H-1} \EE_{s,a \sim d^{\pi^\star}_h}\left[\big(\VV_{P^\star}( V_{h+1; \widehat P}^{\pi^\star}-V_{h+1}^{\pi^\star})\big)(s,a)\right]$, then we have 
\begin{align}
    A\leq 2B+2C,\notag
\end{align}
since $\VV_{P^\star}(a+b)\leq 2\VV_{P^\star}(a)+2\VV_{P^\star}(b)$.

Let $\Delta_{h+1}^{\pi^\star}:=V_{h+1; \widehat P}^{\pi^\star}-V_{h+1}^{\pi^\star}$. Then, w.p. at least $1-\delta$, we have
\begin{align}
    C&=\sum_{h=0}^{H-1} \EE_{s,a\sim d^{ \pi^\star}_h}\left[\big(P^\star(\Delta_{h+1}^{\pi^\star})^2\big)(s,a)-\big(P^\star\Delta_{h+1}^{\pi^\star}\big)^2(s,a)\right]\notag\\
    &=\sum_{h=0}^{H-1} \EE_{s\sim d^{\pi^\star}_{h+1}}\left[(\Delta_{h+1}^{\pi^\star})^2(s)\right]-\sum_{h=0}^{H-1} \EE_{s,a\sim d^{ \pi^\star}_h}\left[\big(P^\star\Delta_{h+1}^{\pi^\star}\big)^2(s,a)\right]\notag\\
    &\leq \sum_{h=0}^{H-1} \EE_{s,a\sim d^{ \pi^\star}_h}\left[(\Delta_{h}^{\pi^\star})^2(s)-\big(P^\star\Delta_{h+1}^{\pi^\star}\big)^2(s,a)\right]\notag\\
    &=\sum_{h=0}^{H-1} \EE_{s,a\sim d^{\pi^\star}_h}\left[\Big((\Delta_{h}^{\pi^\star})(s)+\big(P^\star\Delta_{h+1}^{\pi^\star}\big)(s,a)\Big)\cdot\Big((\Delta_{h}^{\pi^\star})(s)-\big(P^\star\Delta_{h+1}^{\pi^\star}\big)(s,a)\Big)\right],\label{eqn: help 2}
\end{align}
where the first equality is by the definition of variance, the second equality holds as $d^{\pi^\star}_h$ is the occupancy measure also generated under $P^\star$, the first inequality is just changing the index, the third equality holds as $a^2-b^2=(a+b)\cdot (a-b)$. Starting from \pref{eqn: help 2}, we have
\begin{align}
&\sum_{h=0}^{H-1} \EE_{s,a\sim d^{\pi^\star}_h}\left[\Big((\Delta_{h}^{\pi^\star})(s)+\big(P^\star\Delta_{h+1}^{\pi^\star}\big)(s,a)\Big)\cdot\Big((\Delta_{h}^{\pi^\star})(s)-\big(P^\star\Delta_{h+1}^{\pi^\star}\big)(s,a)\Big)\right]\notag\\
     &\leq 2 \sum_{h=0}^{H-1} \EE_{s,a\sim d^{\pi^\star}_h}\left[\left|(\Delta_{h}^{\pi^\star})(s)-\big(P^\star\Delta_{h+1}^{\pi^\star}\big)(s,a)\right|\right]\notag\\
    &=2 \sum_{h=0}^{H-1} \EE_{s,a\sim d^{ \pi^\star}_h}\left[\left|({V}_{h;\widehat P}^{\pi^\star})(s)-\big(P^\star{V}_{h+1;\widehat P}^{\pi^\star}\big)(s,a)-\Big(({V}_{h}^{\pi^\star})(s)-\big(P^\star{V}_{h+1}^{\pi^\star}\big)(s,a)\Big)\right|\right]\notag\\
    &=2 \sum_{h=0}^{H-1} \EE_{s,a\sim d^{ \pi^\star}_h}\left[\left|r(s,a)+\big(\widehat P{V}_{h+1;\widehat P}^{\pi^\star}\big)(s,a)-\big(P^\star{V}_{h+1;\widehat P}^{\pi^\star}\big)(s,a)-r(s,a)\right|\right]\notag\\
        &=2 \sum_{h=0}^{H-1} \EE_{s,a\sim d^{ \pi^\star}_h}\left[\left|\big(\widehat P{V}_{h+1;\widehat P}^{\pi^\star}\big)(s,a)-\big(P^\star{V}_{h+1;\widehat P}^{\pi^\star}\big)(s,a)\right|\right],\label{eqn: help 3}
\end{align}
where the inequality holds as the value functions are all bounded by 1 by the assumption that the total reward over any trajectory is bounded by 1, the first equality is by the definition of $\Delta_{h+1}^{\pi^\star}$, the second equality is because $a$ is drawn from $\pi^\star$. Starting from \pref{eqn: help 3}, we have
\begin{align}
&2 \sum_{h=0}^{H-1} \EE_{s,a\sim d^{ \pi^\star}_h}\left[\left|\big(\widehat P{V}_{h+1;\widehat P}^{\pi^\star}\big)(s,a)-\big(P^\star{V}_{h+1;\widehat P}^{\pi^\star}\big)(s,a)\right|\right]\notag\\
    &= 2 \sum_{h=0}^{H-1} \EE_{s,a\sim d^{ \pi^\star}_h}\left[\left|\EE_{s'\sim P^\star(s,a)}\left[V^{\pi^\star}_{h+1;\widehat P}(s')\right]-\EE_{s'\sim\widehat P(\cdot|s,a)}\left[V^{\pi^\star}_{h+1;\widehat P}(s')\right]\right|\right]\notag\\
        &\leq 16 \sqrt{\sum_{h=0}^{H-1} \EE_{s,a \sim d^{\pi^\star}_h}\left[\big(\VV_{P^\star} V_{h+1; \widehat P}^{\pi^\star}\big)(s,a)\right]\cdot\frac{22C^{\pi^\star}\log(\left|\Pcal\right|/\delta)}{K}}+\frac{880C^{\pi^\star}\log(\left|\Pcal\right|/\delta)}{K}\notag\\
        &=16 \sqrt{A\cdot\frac{22C^{\pi^\star}\log(\left|\Pcal\right|/\delta)}{K}}+\frac{880C^{\pi^\star}\log(\left|\Pcal\right|/\delta)}{K}\label{eqn: help 4}
\end{align}
where the inequality holds with probability at least $1-\delta$ by \pref{lem: offline sum mean value difference}, and the second equality is by definition of $A$.

Then combining \pref{eqn: help 2}, \pref{eqn: help 3} and \pref{eqn: help 4}, we obtain an upper bound for $C$, which suggests
\begin{align}
    A&\leq 2B+2C\notag\\
    &\leq 2B+\frac{1760C^{\pi^\star}\log(\left|\Pcal\right|/\delta)}{K}+32\sqrt{\frac{22C^{\pi^\star}\log(\left|\Pcal\right|/\delta)}{K}}\cdot\sqrt{A}.\notag
\end{align}
Then, with the fact that $x\leq 2a+b^2$ if $x\leq a+b\sqrt{x}$, we have
\begin{align}
    A\leq 4B+\frac{3520C^{\pi^\star}\log(\left|\Pcal\right|/\delta)}{K}+\frac{22528C^{\pi^\star}\log(\left|\Pcal\right|/\delta)}{K}\leq O(B+\frac{C^{\pi^\star}\log(\left|\Pcal\right|/\delta)}{K}).\notag
\end{align}
\end{proof}

With the above lemmas, we can now prove the final results of \pref{thm:mleoffline}. We have that w.p. at least $1-\delta$
\begin{align}
    V_{0;P^\star}^{\pi^\star}-V_{0;P^\star}^{\hat{\pi}}&\leq O\Big(\sqrt{A\cdot\frac{C^{\pi^\star}\log(\left|\Pcal\right|/\delta)}{K}}+\frac{C^{\pi^\star}\log(\left|\Pcal\right|/\delta)}{K}\Big)\notag\\
    &\leq O\Big(\sqrt{(B+\frac{C^{\pi^\star}\log(\left|\Pcal\right|/\delta)}{K})\cdot\frac{C^{\pi^\star}\log(\left|\Pcal\right|/\delta)}{K}}+\frac{C^{\pi^\star}\log(\left|\Pcal\right|/\delta)}{K}\Big)\notag\\
    &\leq  O\Big(\sqrt{B\cdot\frac{C^{\pi^\star}\log(\left|\Pcal\right|/\delta)}{K}}+\sqrt{\frac{C^{\pi^\star}\log(\left|\Pcal\right|/\delta)}{K}\cdot\frac{C^{\pi^\star}\log(\left|\Pcal\right|/\delta)}{K}}+\frac{C^{\pi^\star}\log(\left|\Pcal\right|/\delta)}{K}\Big)\notag\\
    &= O\Big(\sqrt{\sum_{h=0}^{H-1} \EE_{s,a\sim d^{ \pi^\star}_h}\left[\big(\VV_{P^\star} V_{h+1}^{\pi^\star}\big)(s,a)\right]\cdot\frac{C^{\pi^\star}\log(\left|\Pcal\right|/\delta)}{K}}+ \frac{C^{\pi^\star}\log(\left|\Pcal\right|/\delta)}{K}\Big)\label{eqn: intermediate step proof offline}\\
    &=O\Big(\sqrt{\frac{\var_{\pi^\star}C^{\pi^\star}\log(\left|\Pcal\right|/\delta)}{K}}+ \frac{C^{\pi^\star}\log(\left|\Pcal\right|/\delta)}{K}\Big)\notag\,,
\end{align}
where in the last equation we use \pref{lem:variance_lemma}, and $\var_{\pi^\star}:=\EE\left[\bigg(\sum_{h=0}^{H-1}r(s_h,\pi^\star(s_h))-V_0^{\pi^\star}\bigg)^2\right]$.
\end{proof}

\subsection{Proof of \pref{corr:coro_faster}}\label{app:offline_coro_faster}
\begin{proof}[Proof of \pref{corr:coro_faster}]
    By \pref{lem:variance_lemma}, we have
\begin{equation}
    \var_{\pi^*}=\sum_{h=0}^{H-1}\EE_{s,a \sim d^{\pi^*}_h}\left[\big(\VV_{P^\star} V_{h+1}^{\pi^*}\big)(s,a)\right]
\end{equation}
Therefore, when $P^\star$ is deterministic, 
the $\EE_{s,a \sim d^{\pi^*}_h}\left[\big(\VV_{P^\star} V_{h+1}^{\pi^*}\big)(s,a)\right]$ terms are all 0 for any $k\in[K-1]$ and $h\in[H-1]$, and then the $\var_{\pi^*}$ term in the higher order term in \pref{thm:mleoffline} is 0.
\end{proof}

\subsection{Proof of  \pref{corr:offline_coro_infinite}}\label{app:offline_coro_infinite}
\begin{proof}[Proof of \pref{corr:offline_coro_infinite}]
This claim follows the proof of \pref{thm:mleoffline}, while we take a different choice of $\beta$ that depends on the bracketing number and follow the MLE guarantee in \pref{lem:mle_generalization infinite} for infinite model class.
\end{proof}

\subsection{Proof of the claim in \pref{ex: offline coverage}}\label{app: example proof}
\begin{proof}
    Recall that in \pref{def: coverage offline}, we have
    \begin{align*}
C^{\pi^*}_{\Dcal} := \max_{h, P \in \Pcal}  \frac{ \EE_{s,a\sim d^{\pi^*}_h} \mathbb{H}^2\left( P(s,a) \Mid P^\star(s,a) \right)   }{ 1/K \sum_{k=1}^K \mathbb H^2\left(   P(s_h^k,a_h^k)  \Mid   P^\star(s_h^k,a_h^k)  \right)    }\,.
\end{align*} 
For each step $h$, define two distributions, $p_h,q_h$, where $p_h(s,a)=d^{\pi^*}(s,a)$, $q_h(s,a)=\frac{1}{K}\sum_{k=1}^K \mathbb I \{(s,a)=(s_h^k,a_h^k)\}$, and we define $f(s,a,P)=\mathbb H^2(P(s,a)\Mid P^\star(s,a))$, then we have
\begin{align}
    C^{\pi^*}_{\Dcal} &= \max_{h, P \in \Pcal}  \frac{ \EE_{s,a\sim p_h} f(s,a,P)   }{ \EE_{s,a\sim q_h} f(s,a,P)     }\notag\\
    &= \max_{h, P \in \Pcal}  \frac{ \EE_{s,a\sim q_h} \frac{p_h(s,a)}{q_h(s,a)}f(s,a,P)  }{ \EE_{s,a\sim q_h} f(s,a,P)   }\notag\\
    &\leq \max_{h,s,a} \frac{p_h(s,a)}{q_h(s,a)}\notag\\
    &\leq \max_{h,s,a} \frac{1}{q_h(s,a)}\,.
\end{align}
Note that for all $h$, $\{(s_h^k,a_h^k)\}_{k=1}^K$ are i.i.d. samples drawn from $d_h^{\pi^b}$, therefore, $\mathbb E [\mathbb I\{(s_h^k,a_h^k)=(s,a)\}]=d^{\pi^b}_h(s,a)$. By Hoeffding's inequality and with a union bound over $s,a,h$, and for $K\geq \frac{2\log(\frac{|\Scal||\Acal|H}{\delta})}{\rho_{\min}^2}$, w.p. at least $1-\delta$, we have
\begin{align}
    q_h(s,a)&=\frac{1}{K}\sum_{k=1}^K \mathbb I \{(s_h^k,a_h^k)=(s,a)\}\notag\\
    &\geq d^{\pi^b}_h(s,a)-\sqrt{\frac{\log(\frac{|\Scal||\Acal|H}{\delta})}{2K}}\notag\\
    &\geq \frac{d^{\pi^b}_h(s,a)}{2}\,,
\end{align}
where in the last inequality we use the assumption that $d^{\pi^b}_h(s,a) \geq \rho_{\min}, \forall s,a, h$, which gives us $K \geq \frac{2\log(\frac{|\Scal||\Acal|H}{\delta})}{\rho_{\min}^2}\geq\max_{s,a,h}\frac{2\log(\frac{|\Scal||\Acal|H}{\delta})}{(d^{\pi^b}_h(s,a))^2}$, so $K\geq \frac{2\log(\frac{|\Scal||\Acal|H}{\delta})}{(d^{\pi^b}_h(s,a))^2}$ for any $s,a,h$.

Therefore, with $K\geq \frac{2\log(\frac{|\Scal||\Acal|H}{\delta})}{\rho_{\min}^2}$, we have that w.p. at least $1-\delta$
\begin{align}
    C^{\pi^*}_{\Dcal}&\leq \max_{h,s,a} \frac{1}{q_h(s,a)}\leq \max_{h,s,a} \frac{2}{d^{\pi^b}_h(s,a)}\leq \frac{2}{\rho_{min}}\,.
\end{align}

\end{proof}
\end{document}